\documentclass[11pt]{article}
\usepackage[margin=2cm]{geometry}
\usepackage{comment}
\usepackage{amsmath}
\usepackage{multirow}
\usepackage{amsthm}
\usepackage{amssymb}
\usepackage{graphics}
\usepackage{graphicx}
\usepackage{latexsym}
\usepackage{indentfirst}
\usepackage{graphicx}
\usepackage{amsmath}
\usepackage{amssymb}
\usepackage{amsthm}
\usepackage{bm}
\usepackage{relsize}
\usepackage{titling}
\usepackage{titlesec}
\usepackage{tikz}
\usepackage{float}
\usepackage{subcaption}
\usetikzlibrary{arrows}
\usepackage{hyperref}
\usepackage[export]{adjustbox}

\newtheorem{theorem}{Theorem}

\newcommand{\field}[1]{\mathbb{#1}}
\newcommand{\R}{\field{R}}

\begin{document}

\title{Inverse problem for parameters identification in a modified SIRD epidemic model using ensemble neural networks}

\author{
\textbf{Marian Petrica}\\
{\small Faculty of Mathematics and Computer Science}\\
{\small University of Bucharest}\\
{\small Gheorghe Mihoc - Caius Iacob Institute of}\\
{\small Mathematical Statistics and Applied Mathematics}\\
{\small of the Romanian Academy}\\
{\small marianpetrica11@gmail.com}
\and
\textbf{Ionel Popescu}\\
{\small Faculty of Mathematics and Computer Science}\\
{\small University of Bucharest}\\
{\small Institute of Mathematics of the Romanian Academy}\\
{\small ioionel@gmail.com}\\
}
\date{\vspace{-1ex}}
\maketitle

\begin{abstract}  
In this paper, we propose a parameter identification methodology of the SIRD model, an extension of the classical SIR model, that considers the deceased as a  separate category. In addition, our model includes one parameter which is the ratio between the real total number of infected and the number of infected that  were documented in the official statistics.

Due to many factors, like governmental decisions, several variants circulating, opening and closing of schools, the typical assumption that the parameters of the model stay constant for long periods of time is not realistic.  Thus our objective is to create a method which works for short periods of time. In this scope, we approach the estimation relying on the previous 7 days of data and then use the identified parameters to make predictions. 

To perform the estimation of the parameters we propose the average of an ensemble of neural networks. Each neural network is constructed based on a database built by solving the SIRD for 7 days, with random parameters.  In this way, the networks learn the parameters from the solution of the SIRD model.  

Lastly we use the ensemble to get estimates of the parameters from the real data of Covid19 in Romania and then we illustrate the predictions for different periods of time, from 10 up to 45 days, for the number of deaths. The main goal was to apply this approach on the analysis of COVID-19 evolution in Romania, but this was also exemplified on other countries like Hungary, Czech Republic and Poland with similar results.

The results are backed by a theorem which guarantees that we can recover the parameters of the model from the reported data. We believe this methodology can be used as a general tool for dealing with short term predictions of infectious diseases or in other compartmental models.
\end{abstract}

\section{Introduction}

Infectious disease pandemics have had a major impact on the evolution of mankind and have played a critical role in the course of history. Over the ages, pandemics made countless victims, decimating entire nations and civilizations. As medicine and technology have made remarkable progress in the last century, the means of fighting pandemics have become significantly more efficient. A reality nowadays is that globalization, the development of commerce, and the ease to travel all over the world facilitate the transmission mechanism of a new disease much more than it did in the past.

In 2019, Covid19, a virus from the coronavirus family appeared and spread around the world very quickly.  This changed dramatically our world as we know it.  

\subsection{Covid pandemic}

In 2019, COVID19, a virus from the coronavirus family appeared and spread around the world very quickly. The first cases of COVID19 were reported in China in December 2019 and, within less than 3 months, the outbreak became a global pandemic, spreading across almost all countries all over the world. In the period that followed, we witnessed many changes in political decisions including lockdowns, school closures and many other restrictions which were aimed to control the spreading of the virus. 

\subsection{The main ideas of this paper}

Fighting COVID-19 was at first driven by quarantine and other restriction measures.  These were imposed because the mechanisms of infection were not well understood and the hospitals were overwhelmed.  Later on, various degrees of restrictions were imposed in order to control the spread and, at the same time, let the economy recover.  Thus many political decisions affected for better or for worse the transmission of the virus.  

Mathematical modeling is by now one of the scientific pillars on which we build our understanding of the world. It is a valuable tool that can be used in the assessment, prediction and control of infectious diseases, as it is the COVID-19 pandemic.  There is a growing body of mathematical models used at the moment for the spread of virus which are related to our approach in this note.  From the rich literature around this topic we  sample the papers  \cite{choisy2007mathematical, wakefield2019spatio, grassly2008mathematical, kucharski2020early,sardar2020assessment,schuttler2020covid,ferguson2020report,tsay2020modeling}.

One popular choice used to mathematically model the epidemics is the $SIR$ model, appeared in \cite{rossjohn, ross1916application, ross1917application, kermack1991contributions1, kermack1991contributions2, kermack1933contributions3}. This is a compartmental model where an individual can be in one of the following $3$ states, at any given time: susceptible $(S)$, infected $(I)$ or recovered (in some sources removed) $(R)$.  
An extension of the $SIR$ model is the $SIRD$ model which considers the category of deaths as a separate compartment.

The epidemic $SIR/SIRD$ models are non-linear and the identification of the parameters is challenging, compared to the linear models where analytical approaches exist (Godfrey and DiStefano, \cite{godfrey}).  For the case of non-linear models there are methods for identifying the parameters, that are assumed to be constant, by using Taylor series (Gun et al, \cite{gunn}; Pohjanpalo \cite{pohjanpalo}) or differential algebra (Audoly et al, \cite{audoly}; Eisenberg \cite{eisenberg}).  Differential algebra can also be an useful tool in the case of time dependent parameters (Hadeler, \cite{hadeler}, Mummert, \cite{mummert}).  Marinov et al, \cite{marinov}, also proposed a numerical scheme for coefficient identification in SIR epidemic models using the Euler-Lagrange equations.  In the case of the stochastic models, parameter estimation is a type of statistical inference, procedures as least square estimation (Banks et al. \cite{banks}) or maximum likelihood estimation (Julier, \cite{julier}) being the most used techniques. 

Many studies have been done on the evolution of the pandemics around the world. For the case of Covid19 pandemic, the assessment and prediction of the spread as well as the identification of the parameters were rigorously and in detail studied in the papers \cite{lazebnik, bentout1, bentout2, bentout3, bentout4, petrica2020regime, petrica2020self, esquivel2021multi}. 

The purpose of this work is to detail a predictive model, define a methodology of identifying its parameters and accurately assess the transmission dynamic of COVID-19, with potential applicability for other infectious models.  At the same time we analyse the evolution of the pandemic in Romania and we validate the method, in a separate Appendix, for the case of Hungary, Czech Republic and Poland.  The application of predictive models in the study of pandemics dynamics has been exhaustively addressed in the following studies \cite{lazebnik, mohamadou, vega, alexi, rahimi}.

The starting point of our study is the idea that, during the pandemic, it is hard to measure the number of infected, susceptible or recovered persons over time. Due to the large size population, it is very difficult to accurately know the number of infected or the number of recovered people. As Covid19 showed, the governmental institutions and the hospitals became rapidly overwhelmed, while many people became infected and treated themselves at home, without it being recorded in the official statistics.

The dynamic of the pandemic was driven by many factors,  for instance lockdown, opening or closure of the restaurants, elections, opening or closing of schools, vacationing and other measures taken by authorities in order to mitigate the pandemic effects. 
Thus any reasonable parametric model is negatively impacted by the fact that the coefficients are not constant in time.  This implies that the estimations can not be done on long term.  In a previous paper \cite{petrica2020regime} we considered a regime switch which was geared toward adapting the parameters to the political decisions of lockdown or relaxation.  

In this paper the main idea is to use a dynamic model which does not take into account long term evolution or outside assumptions about the status of the pandemic.  In this approach we consider the data on a short period of time, in our case we chose to work with seven days and estimate the parameters relying primarily on the number of deceased people.  

We do this in several steps.  The first one is the cleaning and smoothing of the data.  We fixed some anomalies in the reporting and we take a moving average of two weeks time. The second step is to exploit the model and generate data with parameters chosen at random.  Based on only seven days of data, we train neural networks to learn the parameters of the model.   Furthermore, we used these neural networks to estimate the parameters based on the real data.  The key here is a form of ensemble learning which is interesting in itself.  A single neural network does not seem to predict the parameters very well, however the average has a much better prediction power.  The third and last step is to predict the evolution for a number of future days and compare with the real data.  For a range of ten days, we get very close results to the real data.   

We should point out that in mathematical terms, our approach is a typical inverse problem.  Generically, inverse problems can be ill posed.  We show that our model is actually well posed, meaning that determination of the coefficients from the observation do identify the coefficient uniquely.  This is the backbone of our analysis and estimation.   

The paper is organized as follows.  In Section~\ref{s:3} we briefly discuss the SIR model and we introduce the SIRD model.  Here we state Theorem~\ref{t:2} whose proof is in Appendix 1.  Also here we describe the guiding idea of our approach.  In Section~\ref{s:2} we discuss the data and its cleaning and describe the construction of the neural networks.  Then we proceed with the Discussions and Conclusions in Section~\ref{s:6}.  

\section{The SIR and SIRD models}\label{s:3}

One of the most important models that can describe infectious diseases is the SIR model. The first ones that developed SIR epidemic models were Bernoulli \cite{bernoulli}, Ross \cite{ross1916application,ross1917application}, Kermack-McKendrick \cite{kermack1927contribution} and Kendall \cite{kendall2020deterministic}.

The SIR model is a mathematical model that can be used in epidemiology to analyze, at a given time for a specific population,  the interactions and dependencies between the number of individuals who are susceptible to get an infectious disease, the number of people who are currently infected and those who have already been recovered or have died as a cause of the infection. This model can be used to describe diseases that can be contracted just one time, meaning that a susceptible individual gets a disease by contracting an infectious agent, which is afterward removed (death or recovery).

It is assumed that an individual can be in either one of the following three states: susceptible (S), infected (I) and removed/ recovered (R). This can be represented in the following mathematical schema:

\tikzstyle{int}=[draw, fill=blue!20, minimum size=2em]
\tikzstyle{init} = [pin edge={to-,thin,black}]
\begin{center}
\begin{tikzpicture}[node distance=3 cm,auto,>=latex']
    \node [int] (a) {Susceptible};
    \node [int] (c) [right of=a] {Infected};
    \node [int] (d) [right of=c] {Recovered};
    \node [coordinate] (end) [right of=c, node distance=2cm]{};
    \path[->] (a) edge node {$\beta$} (c);
    \draw[->] (c) edge node {$\gamma$} (end) ;
\end{tikzpicture}
\end{center}

where:
\begin{itemize}
\item $\beta$ = infection rate
\item $\gamma$ = recovery rate.
\end{itemize}

We consider $N$ as the total population in the affected area. We assume $N$ to be fixed, with no births or deaths by other causes, for a given period of n days. Therefore, $N$ is the sum of the three categories previously defined: the number of susceptible people, the ones infected, and the ones removed:
\[
N = \bar{S} + \bar{I} + \bar{R}.
\]

Therefore, we analyze the following SIR model: at time $t$, we consider $\bar{S}(t)$ as the number of susceptible individuals, $\bar{I}(t)$ as the number of infected individuals, and $\bar{R}(t)$ as the number of removed/recovered individuals. The equations of the SIR model are the following:

\begin{equation}\label{eq0}
\begin{cases}
\frac{d \bar{S}(t)}{dt}=-\frac{\bar{\beta} \bar{S}(t) \bar{I}(t)}{N} \\
\frac{d\bar{I}(t)}{dt}=\frac{\bar{\beta} \bar{S}(t) \bar{I}(t)}{N}-\gamma \bar{I}(t) \\
\frac{d\bar{R}(t)}{dt}=\gamma \bar{I}(t)
\end{cases}
\end{equation}
where:
\begin{itemize}
 \item  $\frac{d \bar{S}(t)}{dt}$ is the rate of change of the number of individuals susceptible to the infection over time;
\item $\frac{d\bar{I}(t)}{dt}$ is the rate of change of the number of individuals infected over time;
\item $\frac{d\bar{R}(t)}{dt}$ is the rate of change of the number of individuals recovered over time.
\end{itemize}

Because there is no canonical choice of $N$, we will transform the system \eqref{eq0} by dividing it by $N$ and considering $S(t)=\bar{S}(t)/N$, $I(t)=\bar{I}(t)/N$ and $\hat{R}(t)=\bar{R}(t)/N$.  It is customary to choose $N=10^6$ for convenience but this is just an arbitrary choice.  For instance, analysis on smaller communities or cities involves less than $10^6$, however, $10^6$ is a common choice because countries number their populations in multiples of $10^6$. With these notations, we translate \eqref{eq0} into
\begin{equation}\label{eq1}
\begin{cases}
\frac{d S(t)}{dt}=-\beta S(t)I(t) \\
\frac{d I(t)}{dt}=-\beta S(t)I(t)-\gamma I(t) \\
\frac{d \hat{R}(t)}{dt}=\gamma I(t)
\end{cases}
\end{equation}
where $\beta=\bar{\beta}/N$ and $\gamma$ is the same as in \eqref{eq0}.  Observe that we actually have that $S(t)+I(t)+\hat{R}(t)=S_0+I_0+\hat{R}_0=1$ for all $t\ge0$.  

We use a slight change in the SIR model which accounts for the number of deceased people separately.  The main idea here being that the number of deceased people might be more reliable than other data, as for instance the number of infected.  In the plain vanilla SIR model the recovered and deceased are combined into the single category of recovered.  The idea of the SIRD model is to use the provided data of deceased people in a significant way.

To this aim, we will work with four variables changing in time, namely $S(t)$, $I(t)$, $R(t)$ and $D(t)$ where $R(t)$ is the proportion of recovered and alive people while the $D(t)$ is the proportion of deceased people.  We set the SIRD model as an interaction driven by the system of differential equations:

\begin{equation}\label{sird}
\begin{cases}
\frac{d S(t)}{dt}=-\beta S(t) I(t)\\
\frac{dI(t)}{dt}=\beta S(t) I(t) -(\gamma_1+\gamma_2) I \\
\frac{dR(t)}{dt}=\gamma_1 I(t)\\
\frac{dD(t)}{dt}=\gamma_2 I(t).
\end{cases}
\end{equation}

Notice that in this setup the recovered population bifurcates into recovered ones, accounted by $R$ and the dead ones accounted by $D$.   We also point out that by taking sum of the two factors $\hat{R}(t)=R(t)+D(t)$ above we fall into the classical SIR model.  The reason of accounting for $D(t)$ separately is that the data reports the number of deaths separately and the model above allows to fit the parameters using the data.

Even if the above system is satisfactory to  a certain degree, we would like to point out that in practice, the number of infected as well as the number of recovered is not really known.  The data we have at our disposal reports the number of infected and recovered which are \emph{documented}.  The real number of infected is not really observed.  Therefore we will adjust the model by introducing another parameter $\alpha$ which measures the proportion of \emph{observed} number of infected and recovered.  Therefore, we denote
\[
 \tilde{I}(t)=\alpha I(t) \text{ and }\tilde{R}(t)=\alpha R(t).
\]
In terms of these new quantities, the SIRD model becomes now

\begin{equation}\label{sird_adj}
\begin{cases}
\frac{d S(t)}{dt}=-\frac{\beta}{\alpha} S(t) \tilde{I}(t)\\
\frac{d\tilde{I}(t)}{dt}=\beta S \tilde{I}(t) -(\gamma_1+\gamma_2) \tilde{I}(t) \\
\frac{d\tilde{R}(t)}{dt}=\gamma_1 \tilde{I}(t)\\
\frac{dD(t)}{dt}=\frac{\gamma_2}{\alpha} \tilde{I}(t).
\end{cases}
\end{equation}

We are going to estimate the parameters $\alpha,\beta,\gamma_1,\gamma_2$ from data based on certain number of days.  The advantage of getting an estimate on $\alpha$ is that we can in reality predict the real number of infected people and also the number of recovered people.  In our adjusted model we have
\begin{equation}\label{e:SIRD_ajd_2}
S(t)+\frac{1}{\alpha}(\tilde{I}(t)+\tilde{R}(t))+D(t)=1
\end{equation}
for all times $t\ge0$.  To see this, we start by noticing that summing up all the equations in \eqref{sird}, we get that the derivative of $S(t)+I(t)+R(t)+D(t)$ is constant in time.  Since $S_0+I_0+R_0+D_0=1$, as they represent the proportion of the entire population, we get that the sum $S(t)+I(t)+R(t)+D(t)=1$ and using the definition of $\alpha$ and $\tilde{I},\tilde{R}$ we arrive at \eqref{e:SIRD_ajd_2}.

We present next the main mathematical result, with a proof in the Appendix 1.  This guarantees that the problem is well posed and we can recover the main parameters of the model.

\begin{theorem}\label{t:2} Given the observations of $\tilde{I}(0)=\tilde{I}_0, \tilde{R}(0)=\tilde{R}_0$ and  $D(t)$ for $t=0,1,2,3,4$, we can uniquely determine the parameters $\alpha, \beta,\gamma_1,\gamma_2$.
\end{theorem}

This result is fundamental for our approach.  It shows that given a number of daily observations, at least $5$ days, we can uniquely determine the parameters of the model.  In practice, given any day $k$ of the pandemic, we will use the previous data on a number of days to determine the parameters. \\ \\
The guiding idea:  We can imagine the map from the daily data to the parameters as a function 
\[
\Phi(data_{gen})=(\beta, \gamma_1,\gamma_2,\alpha)\text{ where } data_{gen}=(\tilde{I}_0,\tilde{R}_0, D(0),D(1),\dots,D(4))
\]
generated by \eqref{sird_adj}.   The theorem ensures that this function is well defined on the set 
\[
Data_{gen}=\{(\tilde{I}_0,\tilde{R}_0, D(0),D(1),\dots,D(4)) \text{ solution of \eqref{sird_adj}  for some } (\beta, \gamma_1,\gamma_2,\alpha)\}.
\]
At this point, given an arbitrary data point $data=(\tilde{I}_0,\tilde{R}_0, D(0),D(1),\dots,D(4))$ it is difficult to check that this belongs to $Data_{gen}$.  Thus our goal is to find the best approximation of the $data$ with a point in $Data_{gen}$.  

In general this is achieved using projection methods, for instance non-linear least square, as it is done in \cite{sardar2020assessment}.  In our case we consider an extension problem, rather than the projection method, through the method of neural networks trained on simulated data, using the SIRD model.  These are functions which construct approximations of $\Phi$ defined on the set $Data_{gen}$, that can naturally extend to the whole space, thus can be interpreted as (approximate) extrapolations of $\Phi$ to the whole space.  As we will describe below, one single neural network did not work very well in our numerical experimentation, while an ensemble of neural networks achieve a better performance.  

Numerical simulations show that we get more robust results by considering a larger number of days for the deceased.  We noticed that $7$ days instead of $5$ give more robust results.  Also we exploit $2$ days of data for infected and recovered to strengthen the robustness. 

One of our main findings is that the prediction works very well as it is shown in the next pictures.  For each day $k$ we predict using our method the next 10 days and take the average for each category.  These are plotted along the averages of the real data for 10 days starting from day $k$.

\begin{figure}[H]
  \centering
 \includegraphics[width=1\linewidth]{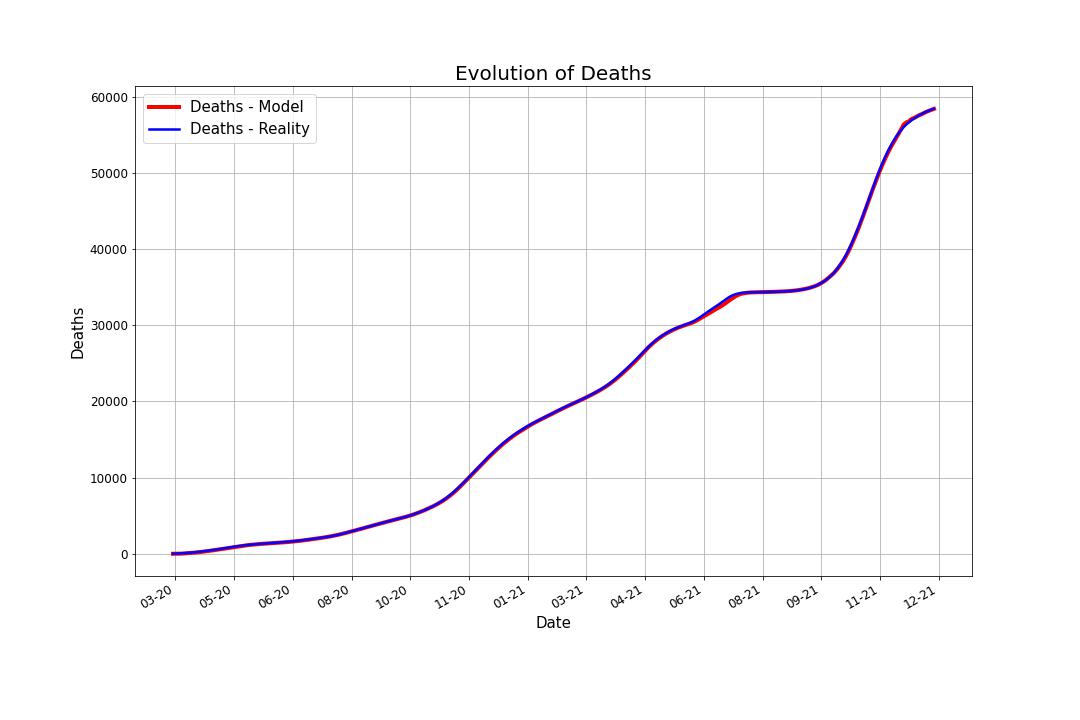}
   \caption{Averages of deaths using 10 future days for predicted and real data.  The red value at time $k$ is computed as the average of future 10 days predictions, while the blue value at time $k$ is the average of the 10 future days values.}
   \label{f:2:1}
\end{figure}

\begin{figure}[H]
  \centering
  \includegraphics[width=0.8\linewidth]{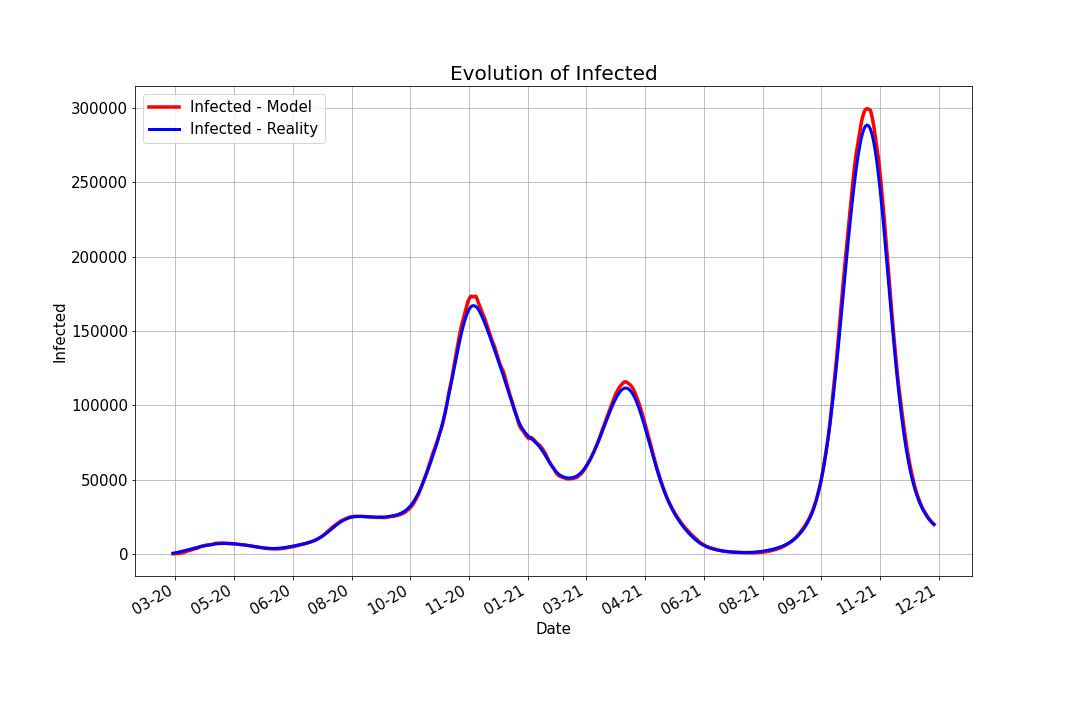}
  \caption{Averages of infected using 10 future days for predicted and real data.  The red value at time $k$ is computed as the average of future 10 days predictions, while the blue value at time $k$ is the average of the 10 future days values.}
   \label{f:2:1:1}
\end{figure}

\begin{figure}[H]
  \centering
 \includegraphics[width=0.8\linewidth]{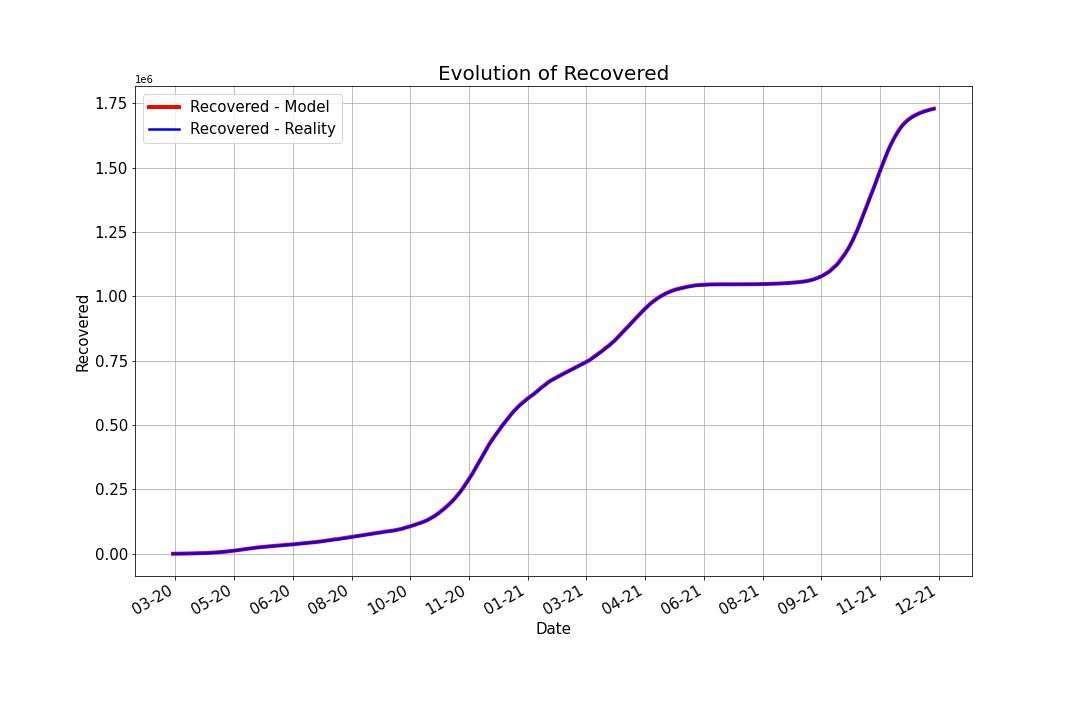}
   \caption{Averages of recovered using 10 future days for predicted and real data.  The red value at time $k$ is computed as the average of future 10 days predictions, while the blue value at time $k$ is the average of the 10 future days values.}
   \label{f:2:1:2}
\end{figure}

We can observe from Figures~\ref{f:2:1}, ~\ref{f:2:1:1}, ~\ref{f:2:1:2}  that the prediction power is excellent for 10 forward days.  By computing the MAE (Mean Absolute Error) for different timeframe predictions and reality data we notice that this observation is validated:

\begin{center}
\begin{tabular}{ |c|c|c|c| } 
\hline
Prediction & Case & MAE \\
\hline
\multirow{3}{*}{10 days prediction } & Deaths & 66.77549 \\ 
& Infected  & 1635.88403 \\ 
& Recovered & 752.07601 \\ 
\hline
\multirow{3}{*}{30 days prediction } & Deaths & 212.02360 \\ 
& Infected  & 13731.60866 \\ 
& Recovered & 4733.52849 \\ 
\hline
\multirow{3}{*}{45 days prediction } & Deaths & 531.00102 \\ 
& Infected  & 32418.20868 \\ 
& Recovered & 14693.97709 \\ 
\hline
\end{tabular}
\end{center}

We detail and discuss more on these predictions in the next Section.

\section{The Data and the Neural Networks}\label{s:2}

In this section we present our strategy for the data cleaning and the estimation of the parameters $\beta,\gamma_1,\gamma_2,\alpha$ of the SIRD model. 

We clean the raw data which has several anomalies and we use a regularization by averaging.

The basic strategy is the following.  Based on the assumed SIRD model we generate data for $7$ days and then train a neural network which learns the parameters from the data for this $7$ days time interval.

Then, using the real data and the neural network we  find the parameters in a dynamical way for any $7$ consecutive days.  Given these $7$ days we can predict based on the model what is going to happen on the next few days.  

In a real world the parameters do not stay constant, they change dynamically and we would like to catch part of this behavior. 

\subsection{Data}\label{s:data}

We took the data from \url{https://datelazi.ro} which keeps a record of all the data during the pandemic in Romania.  

During the pandemic, the reported numbers and the methodology regarding the reporting changed several times causing delays or bad reporting.  In October 2020, the definition of a recovered person changed thus causing a data anomaly in the reported number of recovered. Particularly, we can see a spike of $44000$ new cases from one day to another, equivalent to the cumulative number of cases until then.  To alleviate this anomaly, we distribute the extra number of cases proportionally to the previous days.  

There are also periods of time in which the number of recovered people is actually $0$ for almost three weeks.  

We further analyse the data of the 102 weeks taken into assessment by day of the week and we compute the average of the reported number of infected.

\begin{center}
\begin{tabular}{|l|l|}
\hline
\multicolumn{1}{|c|}{\textit{Day of the week}} & \multicolumn{1}{c|}{\textit{Average number of infected (reported)}} \\ \hline
Monday                              & 2461                                  \\ \hline
Tuesday                             & 4479                                  \\ \hline
Wednesday                           & 4525                                  \\ \hline
Thursday                            & 4453                                  \\ \hline
Friday                              & 4269                                  \\ \hline
Saturday                            & 4041                                  \\ \hline
Sunday                              & 2909                                  \\ \hline
\end{tabular}
\end{center}

We notice a significant difference between the average numbers reported on weekdays and weekends and we perform an \emph{One Way Anova Test} to validate this observation.\\

\begin{tabular}{|l|l|l|llll}
\hline
\multicolumn{1}{|c|}{\textit{Source   of Variation}} & \multicolumn{1}{c|}{\textit{SS}} & \multicolumn{1}{c|}{\textit{df}} & \multicolumn{1}{c|}{\textit{MS}} & \multicolumn{1}{c|}{\textit{F}} & \multicolumn{1}{c|}{\textit{P-value}} & \multicolumn{1}{c|}{\textit{F crit}} \\ \hline
Between Groups                                       & 4.32E+08                         & 6                                & \multicolumn{1}{l|}{72029925}    & \multicolumn{1}{l|}{2.286333}   & \multicolumn{1}{l|}{0.034118}         & \multicolumn{1}{l|}{2.111386}        \\ \hline
Within Groups                                        & 2.23E+10                         & 707                              & \multicolumn{1}{l|}{31504561}    & \multicolumn{1}{l|}{}           & \multicolumn{1}{l|}{}                 & \multicolumn{1}{l|}{}                \\ \hline
                                                     &                                  &                                  &                                  &                                 &                                       &                                      \\ \cline{1-3}
Total                                                & 2.27E+10                         & 713                              &                                  &                                 &                                       &                                      \\ \cline{1-3}
\end{tabular}\\

We found a statistically significant difference between the averages of the reported cases according to the day of the week ($p < 0.05$).  However, by law, the methodology of reporting from the health centers allows reporting cases within an interval of two weeks.  In order to mitigate the above deficiencies, we replaced the data at time $t$ with the average of the data during the previous two weeks preceding time $t$.  We present in Figure~\ref{i:2} the data before and after the cleaning.  

\begin{figure}[H]
  \centering
 \includegraphics[width=\linewidth]{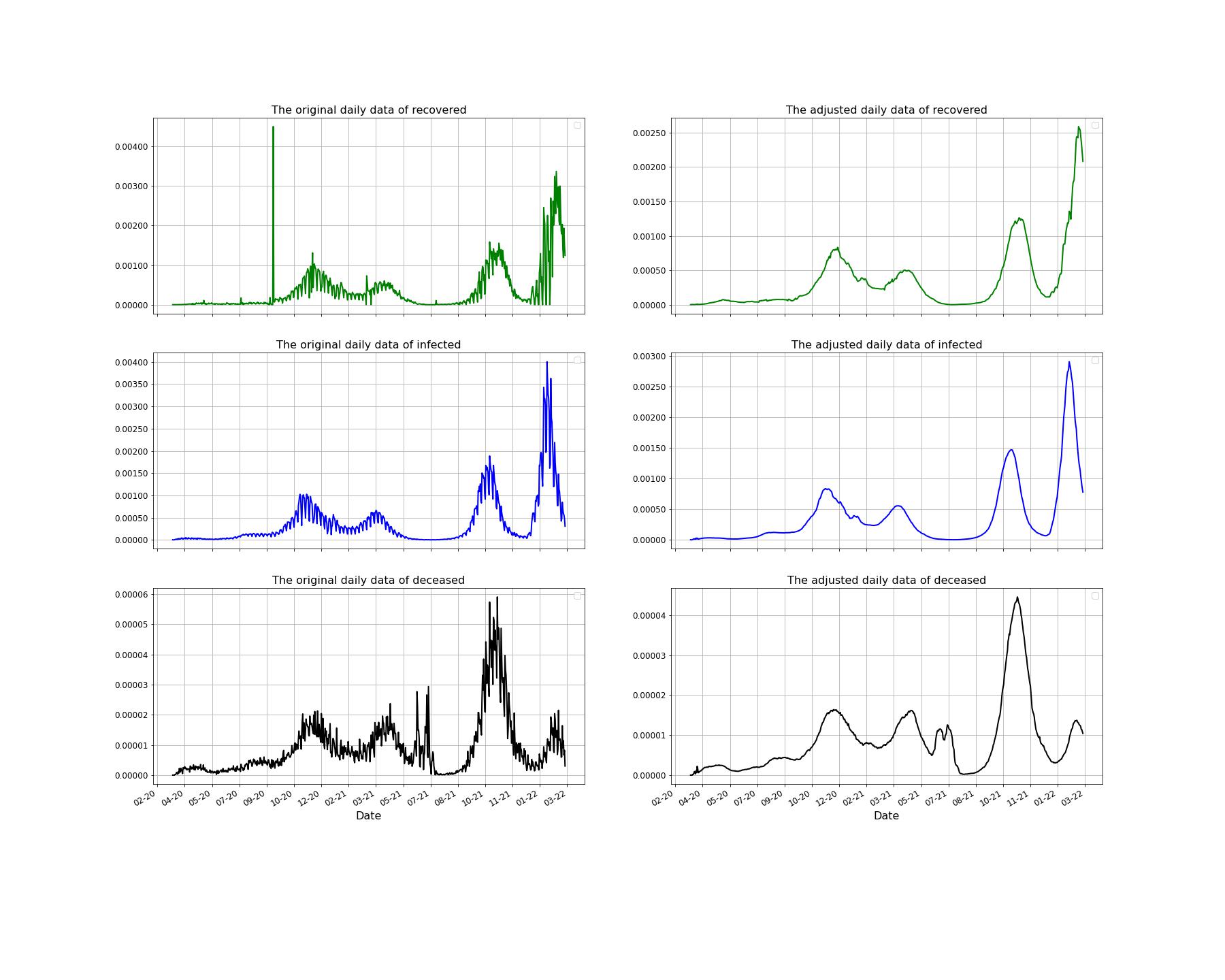}
   \caption{  Data adjustments according to the methodology of reporting the Covid19 numbers in Romania.  The rows describe the data (from top to bottom) of recovered, infected and dead.  The left column represents the raw data and the right column represents the adjusted data as we described above.  Notice the scale and the spike in the first picture which is adjusted as we pointed out.  The data we work with is scaled by $10,000,000$.}
   \label{i:2}
\end{figure}

\subsection{The neural networks}  

To deal with the estimation of the parameters, we follow two main steps. 

The first one is the generation of data. 

We take the range of $\beta$ in the interval $J_\beta=(0,1)$, for $\gamma_1$ we consider the range in $J_{\gamma_1}=(0,1)$, for $\gamma_2$ we consider the interval $J_{\gamma_2}=(0,0.01)$ and for $\alpha$ we consider the interval $J_{\alpha}=(0.01,1)$.  Next we split each of these intervals into $7$ sub-intervals of the same size which we will index accordingly as

$$J_{\beta,i}=(i/7,(i+1)/7)$$ 
$$J_{\gamma_1,i}=(i/7,(i+1)/7))$$ 
$$J_{\gamma_2,i}=(0.01*i/7,0.01*(i+1)/7)$$ 
$$J_{\alpha, i}=(0.01+0.99*i/7,0.01+0.99*(i+1)/7)$$
for $i\in \{0,1,\dots, 6 \}$. The splitting is motivated by the fact that we want to have good representation of the parameters and at the same time we want to avoid concentration of the parameters in one single region. We tried previously to use a simple uniform choice for each parameter in the whole interval, but we run into the problem of misrepresentation of small values of the parameters. It seems that this phenomena is due to some form of concentration of measure which is alleviated by using this splitting method.  With this strategy we also avoid the overfitting problem of the neural networks.  Therefore we obtain 7 sub-intervals for each of the 4 parameters which means that we get $7^4=2401$ combinations of sub-intervals. 

Next, to generate the data we apply the following procedure:    
\begin{enumerate}
    \item Create the data set $\Delta$ to store the values obtained in the next steps
    \item For $i_1\in \{0,1,\dots,6\}$, $i_2\in \{0,1,\dots,6\}$, $i_3\in \{0,1,\dots,6\}$, $i_4\in \{0,1,\dots,6\}$:
    \begin{enumerate}
    \item For $j \in\{1,\dots, 10000\}$ pick at random 
    \begin{enumerate}
        \item $\beta\in J_{\beta, i_1}$, 
        \item $\gamma_1\in J_{\gamma_1,i_2}$, 
        \item $\gamma_2\in J_{\gamma_2,i_3}$ 
        \item $\alpha\in J_{\alpha,i_4}$, 
        \item $I_0$ in the interval $(0,0.2)$, 
        \item $R_0$ in $(0,0.6)$ 
        \item $D_0$ in $(0,0.007)$ 
    \end{enumerate}
    
    \item Solve the system \eqref{sird_adj} with all the parameters from step 1 and 2 for the time interval $[0,7]$ and add the row of $(I_0,I_1,\dots, I_7,R_0,R_1,\dots,R_7,D_0,D_1,\dots, D_7)$ to $\Delta$.
    \end{enumerate}
\end{enumerate}

For each $i\in \{1,2,\dots,10\}$, we create a training sample $B_i$ from the dataset $\Delta$ of size $70\%$ chosen at random without replacement.  Using $B_i$ sample we train a neural network, $\mathrm{NN}_i$, $i\in \{1,2,\dots,10\}$, having as input:
\[
XTrain=(I_0,I_1,I_2,R_0,R_1,R_2,D_0,D_1,D_2,D_3,D_4,D_5,D_6,D_7)
\]
and our parameters
\[
YTrain=(\beta,\gamma_1,\gamma_2,\alpha)
\]
as output.

According to our Theorem~\ref{t:2} we know that we can recover the parameters $\beta, \gamma_1,\gamma_2, \alpha$ only from  $I_0,R_0,D_0,D_1,D_2,D_3,D_4$, from our dataset.  However, we choose to use more data because the estimates are more robust.  This choice can also be interpreted as a regularization which decreases the training/test loss.  

The next step is the construction of the neural networks.  We performed various tests, with different type of architectures, ones with multiple hidden layers and large number of neurons, as well as some architectures with 1-2 hidden layers and small number of neurons.  Based on these tests we draw the conclusion that the large ones are expensive to train while the small ones do not perform very well.  We decided to mitigate these disadvantages by choosing the below architecture, which is a medium one when it comes to its complexity. It helps us achieve great results at good performance with a reasonable consumption of resources. The training was performed on an Intel(R) Core(TM) i9-10885H CPU @ 2.40GHz x 16. The architecture of the neural network we use is of the following form:
\begin{itemize}
    \item Layers:
        \begin{enumerate}
        \item Dense 64, activation function 'ReLU', input dimentsion=14 
        \item Dense 128, activation function 'ReLU'
        \item Dense 256, activation function 'ReLU'
        \item Dense 512, activation function 'ReLU'
        \end{enumerate}
    \item One output for each parameter, $(\beta, \gamma_1,\gamma_2,\alpha)$.
    \item Loss: Mean Absolute Error
    \item Optimizer: Adam (Kingma and Ba, \cite{adam_ml} )
\end{itemize}

The size of the training, respectively test data split for $\mathrm{NN}_i$ is $80\%$, respectively $20\%$ from the sample $B_i$.

After training the neural networks, the predictions of our parameters is made by averaging the predictions from all the individual neural networks on the real data

$$(\hat{\beta},\hat{\gamma_1},\hat{\gamma_2},\hat{\alpha})=\frac{1}{10}\sum_{i=1}^{10}\mbox{NN}_i(data)$$

With this approach, we can achieve better performance of our model because we manage to decrease the variance, without increasing the bias. Usually, the prediction of a single neural network is sensitive to noise in the training set, while the average of many neural networks that are not correlated, is not sensitive.  Bootstrap sampling is a good method of de-correlating neural networks, by training them with different training sets. If we train many neural networks on the same dataset, we will obtain strongly correlated neural networks.

In Figure~\ref{i:4} we present the results of the estimated parameters $\beta, \gamma_1,\gamma_2,\alpha$.

\begin{figure}[H]
  \centering
  \begin{subfigure}[b]{0.48\linewidth}
   \includegraphics[width=\linewidth]{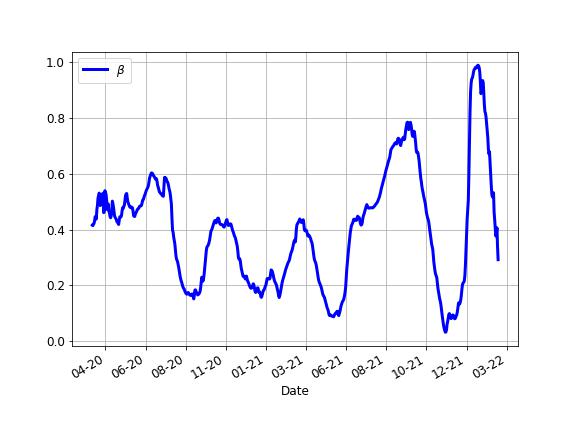}
  \end{subfigure}
  \begin{subfigure}[b]{0.48\linewidth}
   \includegraphics[width=\linewidth]{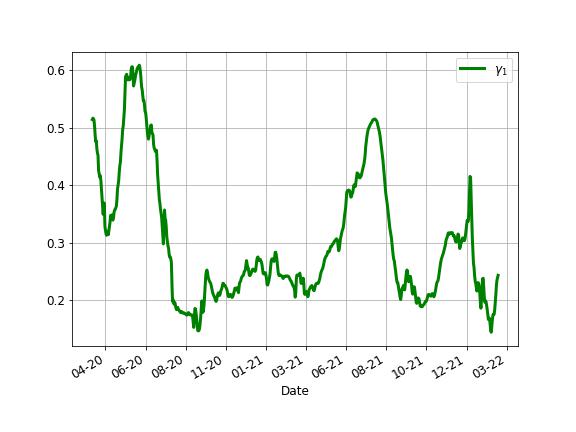}
  \end{subfigure}
  \begin{subfigure}[b]{0.48\linewidth}
   \includegraphics[width=\linewidth]{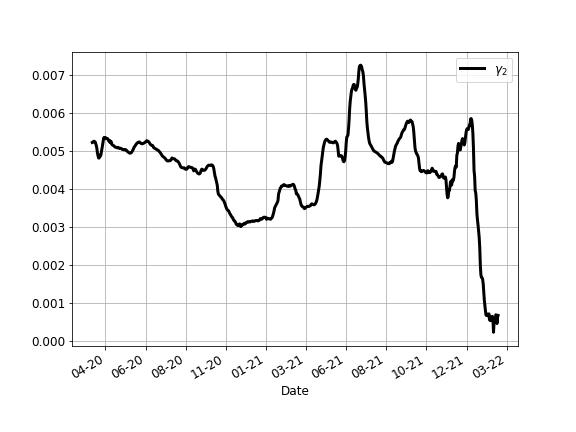}
  \end{subfigure}
  \begin{subfigure}[b]{0.48\linewidth}
   \includegraphics[width=\linewidth]{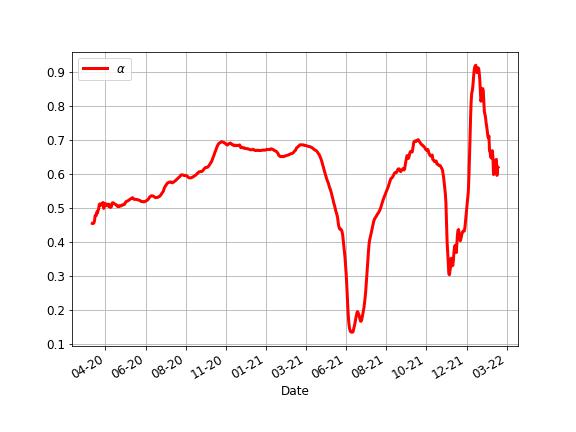}
  \end{subfigure}
  \caption{Parameters of the spreading of Covid19 in Romania estimated using the above averaging of the neural networks.  Notice the behavior of the parameters $\beta, \gamma_1,\gamma_2$ which tend to decrease over the period of almost two years.  Interestingly we see the parameter $\alpha$ having large values during the Fall of 2020 and lower values during the summer of 2021.  This suggests that the proportion of real infected people is between $[1/0.9, 1/0.1]=[1.11,10]$ to the reported infected.  This show that roughly only half of the infected get reported.}
   \label{i:4}
\end{figure}

Knowing the parameters for the model $(\beta_k,\gamma_{1,k},\gamma_{2,k},\alpha_k)$ at each time $k$ and the values $(I_k,R_k,D_k)$, from the real data, we can generate the predictions $P_{k,0},\dots,P_{k,10}$ using the system \eqref{sird_adj} for the time interval $[k,k+10]$. We take the average of $P_{k,0},\dots,P_{k,10}$ and call this $P_k$. 

In Figure~\ref{f:2:2} we plot for each day $k$ the average of the real (smoothed) data for 10 days starting with $k$ alongside with the average $P_k$ computed above.  As we already pointed out, the fit is very good.  

\begin{figure}[H]
  \centering
  \begin{subfigure}[b]{0.49\linewidth}
   \includegraphics[width=\linewidth]{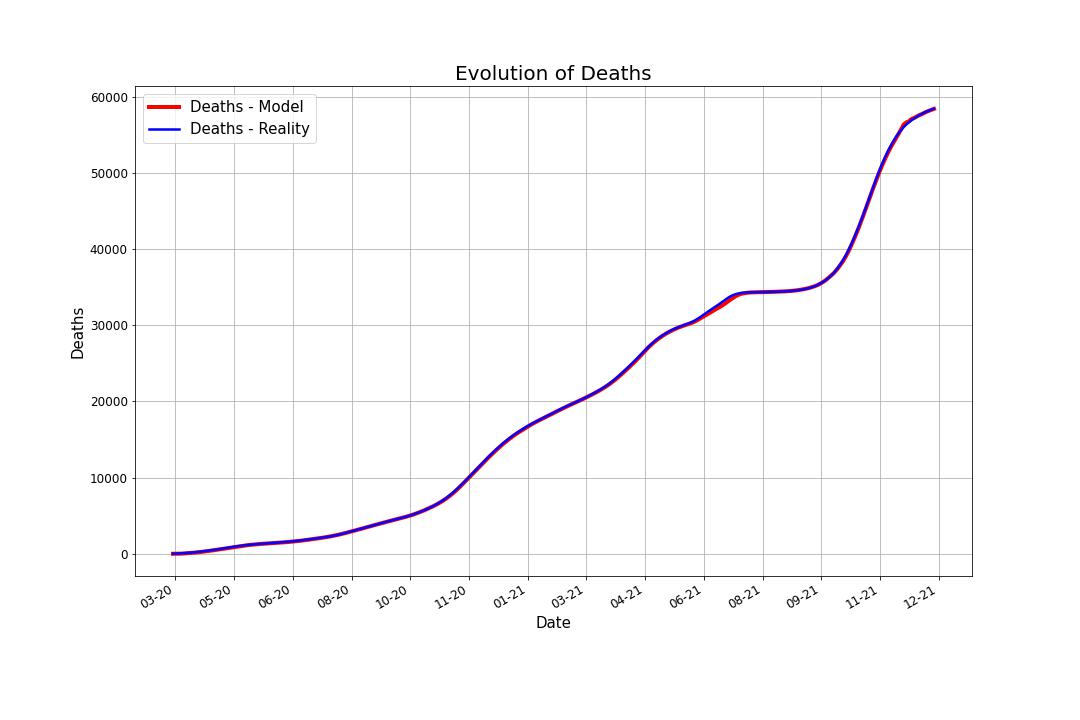}
   \caption{Predicted and real deaths.}
  \end{subfigure}
  \begin{subfigure}[b]{0.49\linewidth}
   \includegraphics[width=\linewidth]{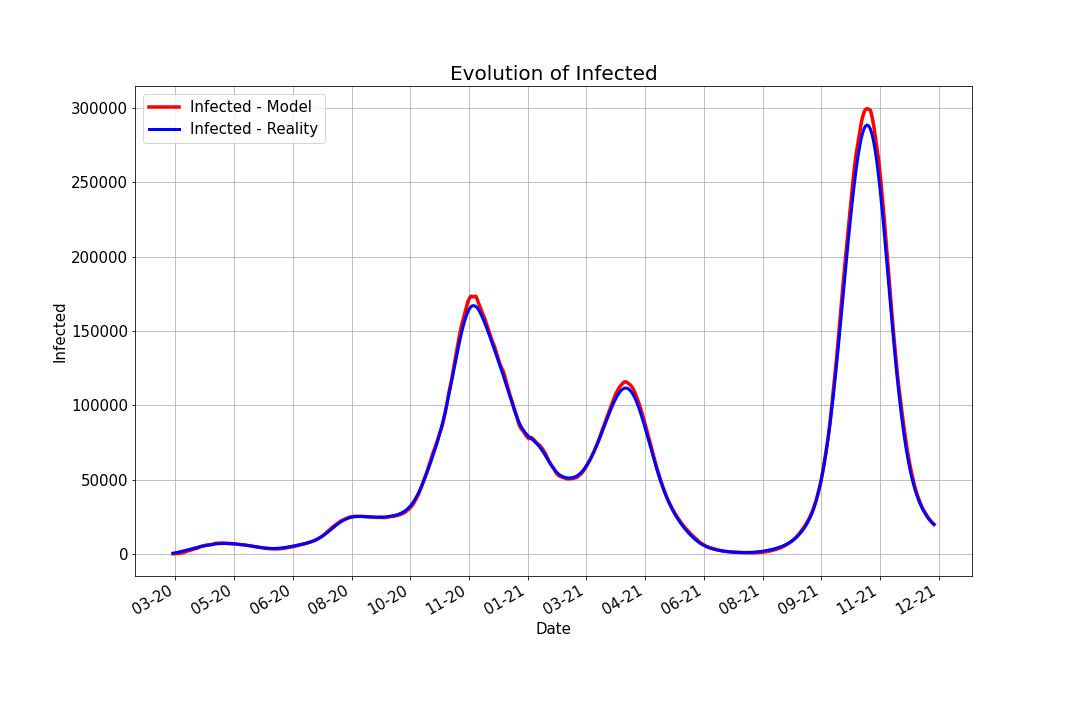}
   \caption{Predicted and real infected.}
  \end{subfigure}
  \begin{subfigure}[b]{0.49\linewidth}
   \includegraphics[width=\linewidth]{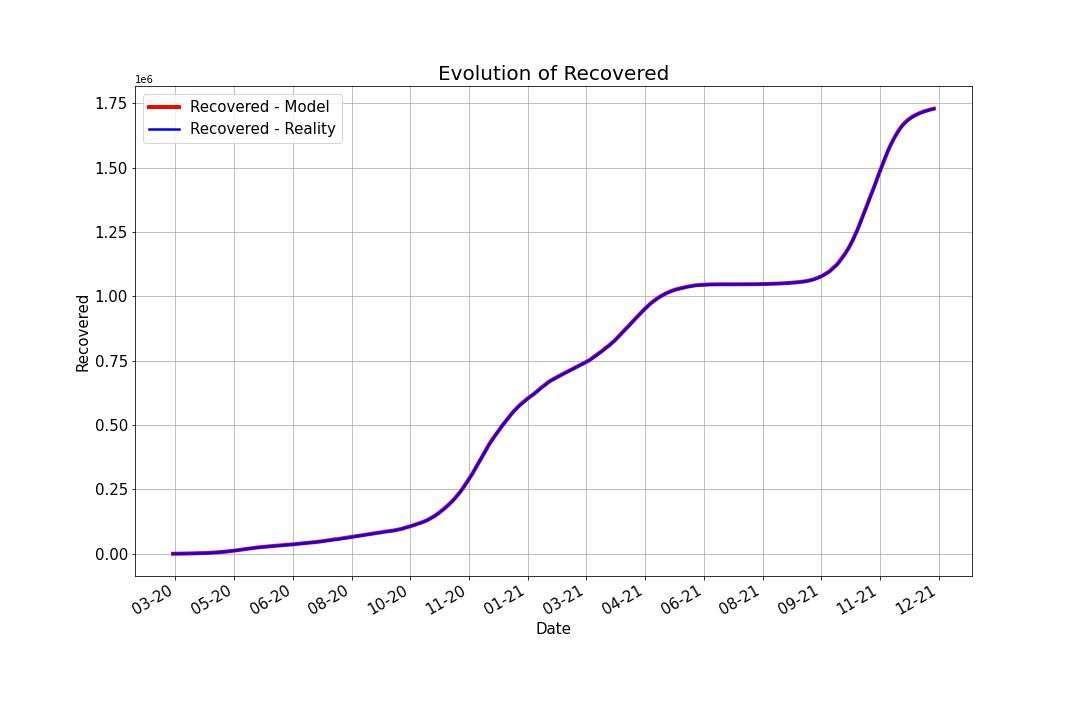}
    \caption{Predicted and real recovered.}
  \end{subfigure}
  \caption{The plots of the real data and the predicted averages for the next 10 days.  As we described above, for each day $k$ we compute the average of the real data for the next $10$ days and the average of the predicted data for the next $10$ days.  Notice the important fact that each prediction is made in terms of the previous $7$ days.  The close match suggests a very good prediction power of our approach.  A slight difference appears in the case of the infected number of people during the forth wave of the pandemic, namely the Fall of 2021.}
  \label{f:2:2}
\end{figure}

The next images, in Figure~\ref{f:2:3}, show the prediction on the death for 30 and 45 days.  In many cases the prediction is good, however there are regions in which the prediction ceases to be accurate. This highly depends on the timeframe we chose to make the predictions.

\begin{figure}[H]
  \centering
  \begin{subfigure}[b]{0.9\linewidth}
   \includegraphics[width=\linewidth]{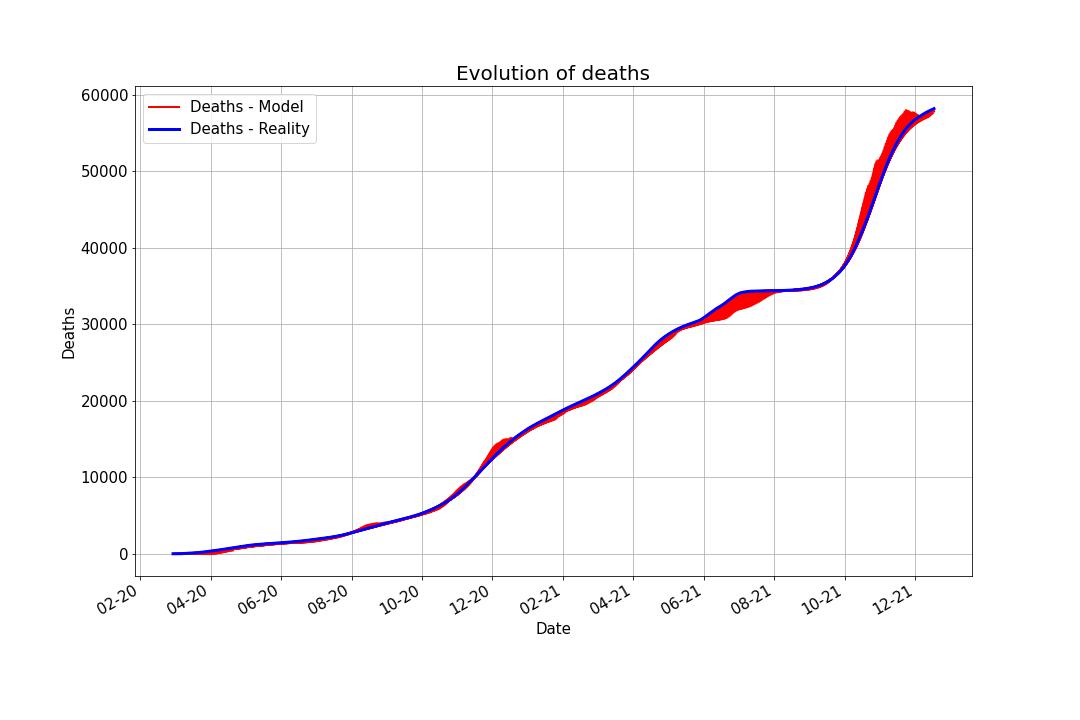}
   \caption{For each day we predicted the deaths for 30 days.  }
  \end{subfigure}
 \begin{subfigure}[b]{0.9\linewidth}
   \includegraphics[width=\linewidth]{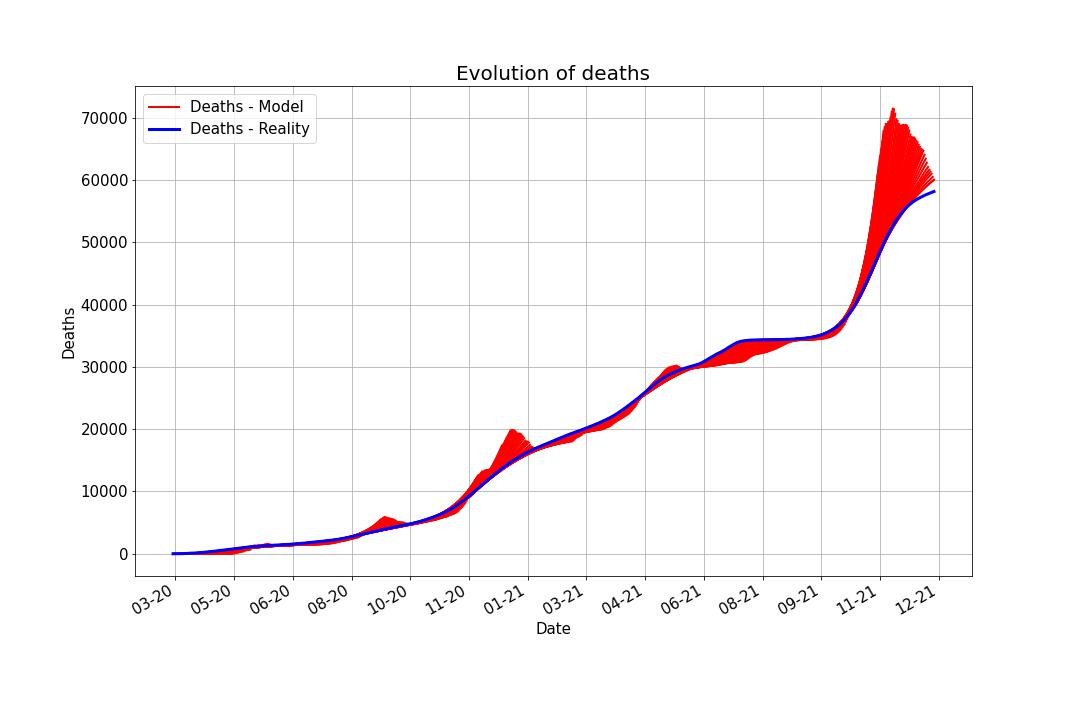}
   \caption{For each day we predicted the deaths for 45 days.  }
  \end{subfigure}
  \caption{In both pictures, in blue is the real (reported) number of deaths.  For each day $k$, we plotted, in red, the prediction of the deaths, starting with day $k$. The first picture shows the prediction is plotted for 30 days, while the second shows the predictions plotted for 45 days.  Remark that the 30 days prediction is much better than the 45 days prediction.}  
    \label{f:2:3}

\end{figure}

Next, in Figures~\ref{f:2:5},\ref{f:2:6}, we look in more details at the images above, to see the refined structure of the behavior.  We do this for 10 days versus 30 days starting at different moments of time.

\begin{figure}[H]
  \centering
  \begin{subfigure}[b]{0.32\linewidth}
   \includegraphics[width=\linewidth]{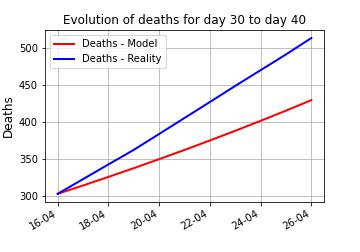}
  \end{subfigure}
  \begin{subfigure}[b]{0.32\linewidth}
   \includegraphics[width=\linewidth]{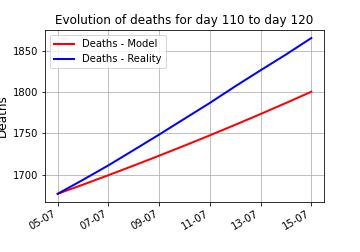}
  \end{subfigure}
  \begin{subfigure}[b]{0.32\linewidth}
   \includegraphics[width=\linewidth]{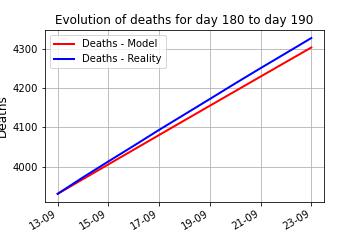}
  \end{subfigure}
  \begin{subfigure}[b]{0.32\linewidth}
   \includegraphics[width=\linewidth]{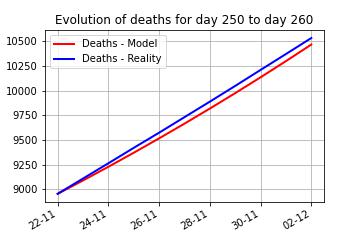}
  \end{subfigure}
  \begin{subfigure}[b]{0.32\linewidth}
   \includegraphics[width=\linewidth]{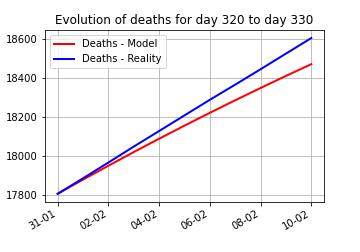}
  \end{subfigure}
  \begin{subfigure}[b]{0.32\linewidth}
   \includegraphics[width=\linewidth]{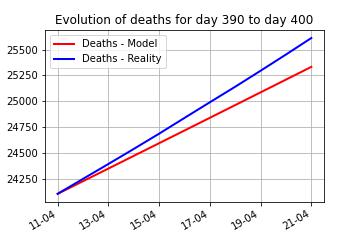}
  \end{subfigure}
  \begin{subfigure}[b]{0.32\linewidth}
   \includegraphics[width=\linewidth]{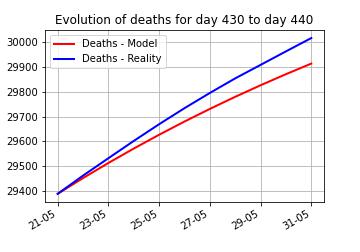}
  \end{subfigure}
  \begin{subfigure}[b]{0.32\linewidth}
   \includegraphics[width=\linewidth]{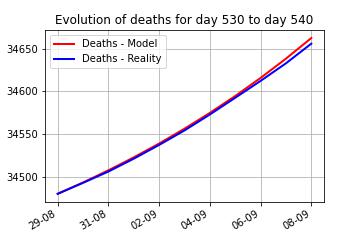}
  \end{subfigure}
    \begin{subfigure}[b]{0.32\linewidth}
   \includegraphics[width=\linewidth]{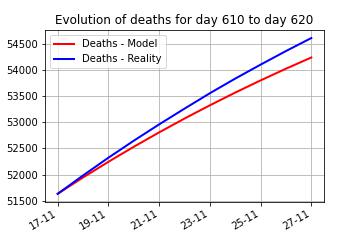}
  \end{subfigure}
  \caption{The evolution of the predicted 10 days starting with the days 30, 110, 180, 250, 320, 390, 430, 530, 610}
  \label{f:2:5}
\end{figure}

\begin{figure}[H]
  \centering
  \begin{subfigure}[b]{0.32\linewidth}
   \includegraphics[width=\linewidth]{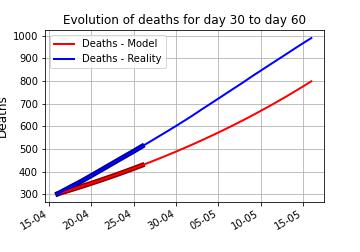}
  \end{subfigure}
  \begin{subfigure}[b]{0.32\linewidth}
   \includegraphics[width=\linewidth]{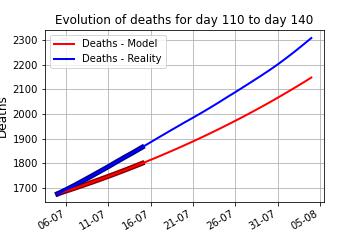}
  \end{subfigure}
  \begin{subfigure}[b]{0.32\linewidth}
   \includegraphics[width=\linewidth]{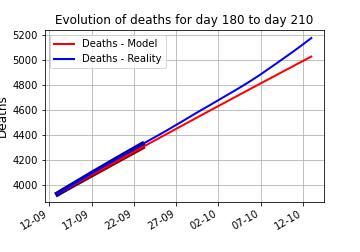}
  \end{subfigure}
  \begin{subfigure}[b]{0.32\linewidth}
   \includegraphics[width=\linewidth]{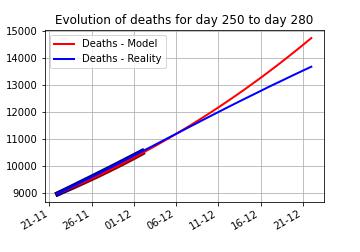}
  \end{subfigure}
  \begin{subfigure}[b]{0.32\linewidth}
   \includegraphics[width=\linewidth]{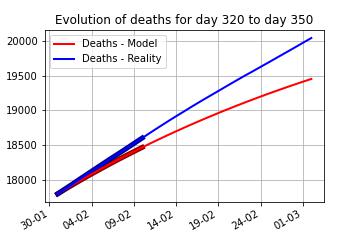}
  \end{subfigure}
  \begin{subfigure}[b]{0.32\linewidth}
   \includegraphics[width=\linewidth]{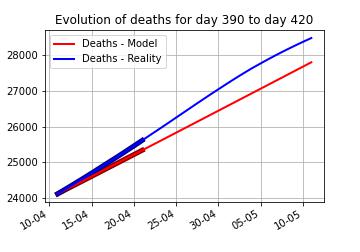}
  \end{subfigure}
  \begin{subfigure}[b]{0.32\linewidth}
   \includegraphics[width=\linewidth]{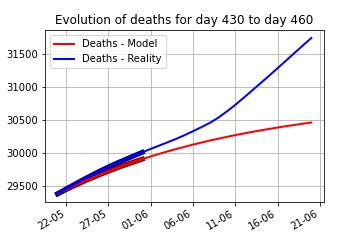}
  \end{subfigure}
  \begin{subfigure}[b]{0.32\linewidth}
   \includegraphics[width=\linewidth]{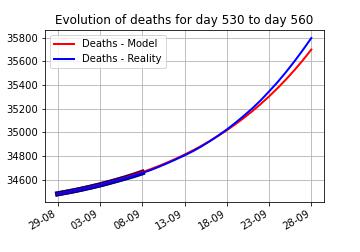}
  \end{subfigure}
    \begin{subfigure}[b]{0.32\linewidth}
   \includegraphics[width=\linewidth]{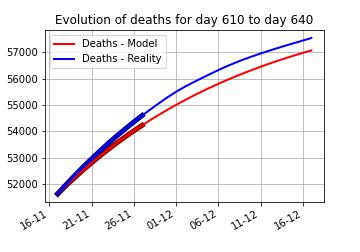}
  \end{subfigure}
  \caption{The evolution of the predicted 30 days starting with the days 30, 110, 180, 250, 320, 390, 430, 530, 610. The predictions for the first 10 days were highlighted. This reinforces that the predictions are good for short periods of time and they loose the prediction power on longer periods of time.  The main lesson we learn from the above figure is the fact that the parameters are not constant in time.}
  \label{f:2:6}
\end{figure}

\section{Discussion and Conclusions}\label{s:6}

The dynamic of an infectious disease is highly impacted by numerous factors, including the measures imposed by governments or the attitude of population towards it, as the COVID-19 pandemic has shown. Therefore, it is very unlikely that the parameters of a model designed to asses the spread of a virus are constant over time. We have to be aware of the fact that, in the long term, the prediction is affected by all the restrictions/relaxations taken by most of the countries.  We believe that this methodology has a high degree of generality to be used in many other cases of prediction, particularly useful for those cases where the prediction depends on many other factors which change the behavior of the model. 

We introduced on our model an extra parameter which accounts for the proportion of the infected and reported population versus the whole infected population.  The point being that not every infected person is actually tested or reported.  Thus the actual number of infected people should be higher.

We do not account for the vaccination campaign, though this does not affect our model since we are looking at the parameters on relatively short periods of time.  The vaccination should in principle change the parameters, which is in fact exactly what we look for. 

Regarding the limitations of the methodology presented in this paper, we can mention the followings. When it comes to the applicability of the model, one limitation could be that the number of recovered people is not included in the reporting by all of the countries, which would make the data incomplete.  In the same time, one other limitation is caused by the changes that could appear in the reporting methodology of a specific country, such as the change of the definitions of infected/ recovered, that can have an impact in the results. 

In order to test if this technique can be well generalized, we applied it to Covid19 data of 3 other countries: Hungary, Czech Republic and Poland.  By replicating the approach, similar to Romania case, we are confident that the predictive model that we presented in this paper can be also applied to other countries in order to identify the parameters of the model and to accurately assess the transmission dynamic of the pandemic, other infectious diseases or other compartmental models. 

\section{Declaration}

\subsection{Ethics approval} We did not use any confidential data for the analysis in this paper and we do not have any ethical issues in this paper.

\subsection{Consent for publication}

We did not use any data which could possibly reveal any personal data of any patient.

\subsection{Availability of data and material} We used the public data from \href{https://datahub.io/core/covid-19/r/3.html}{here}.

\subsection{ Competing interests } The authors declare that they have no competing interests.

\subsection{ Funding } There are no funding sources for this paper.

\subsection{Authors' contributions}  All authors contributed equally to this paper.

\subsection{Acknowledgement}  The last author would like to thank Iulian Cimpean, Lucian Beznea and Mihai N. Pascu for interesting discussions about this paper.

\section{Appendix 1}\label{a:1}

The goal of this section is to provide the proof of Theorem~\ref{t:2}. Recall the system \eqref{sird_adj} given by

\begin{equation}\label{sird_adj_a}
\begin{cases}
\frac{d S(t)}{dt}=-\frac{\beta}{\alpha} S \tilde{I}\\
\frac{d\tilde{I}(t)}{dt}=\beta S \tilde{I} -(\gamma_1+\gamma_2) \tilde{I} \\
\frac{d\tilde{R}(t)}{dt}=\gamma_1 \tilde{I}\\
\frac{dD(t)}{dt}=\frac{\gamma_2}{\alpha} \tilde{I}.
\end{cases}
\end{equation}

The statement of Theorem~\ref{t:2} is the following.

\begin{theorem} Given $\tilde{I}(0), \tilde{R}(0)$ and  $D(0),D(1),D(2),D(3),D(4)$ we can uniquely determine the parameters $\alpha, \beta,\gamma_1,\gamma_2$ of \eqref{sird_adj_a}.
\end{theorem}

\begin{proof}

We will first reduce the analysis to a single equation, namely the one for $D(t)$.  To do this we will write each of the involved quantities as functions of $D(t)$ as follows
\begin{equation*}
S=u(D), \tilde{I}=v(D), \tilde{R}=w(D).
\end{equation*}
The easiest to deal with is $\tilde{R}$ because from the last two equations we get
\begin{equation*}
\frac{d\tilde{R}(t)}{dt}=\frac{\gamma_1}{\alpha\gamma_2}\frac{dD(t)}{dt}
\end{equation*}
which leads to $\tilde{R}(t)=\frac{\gamma_1}{\alpha\gamma_2}(D(t)-D_0)+\tilde{R}_0$.

Now, we treat the function $u$ which determines $S(t)=u(D(t))$.  Dividing the first and the last from \eqref{sird_adj_a} we get
\begin{equation*}
u'(D)=-\frac{\beta}{\gamma_2}u(D)
\end{equation*}
which can be integrated to give $S(t)$ in terms of $D(t)$ as
\begin{equation}\label{e:S_s}
S(t)=S_0\exp\left( -\frac{\beta}{\gamma_2}(D(t)-D_0) \right).
\end{equation}
Furthermore, this allows us to solve for $\tilde{I}(t)=v(D(t))$.  To see this, add the first two equations from \eqref{sird_adj_a} and then combine this with the last one to arrive at
\begin{equation*}
\frac{dS}{dt}(t)+\frac{1}{\alpha}\frac{d\tilde{I}}{dt}(t)=-\frac{(\gamma_1+\gamma_2)}{\alpha}\tilde{I}(t)=-\frac{\gamma_1+\gamma_2}{\gamma_2}\frac{dD}{dt}(t)
\end{equation*}
from which we deduce that
\begin{equation*}
S(t)+\frac{1}{\alpha}\tilde{I}(t)+\frac{\gamma_1+\gamma_2}{\gamma_2}D(t)=S_0+\frac{1}{\alpha}\tilde{I}_0+\frac{\gamma_1+\gamma_2}{\gamma_2}D_0.
\end{equation*}
Solving now for $\tilde{I}$ and using \eqref{e:S_s} we obtain
\[
 \tilde{I}(t)=\alpha S_0 +\tilde{I}_0 +\frac{\alpha(\gamma_1+\gamma_2)}{\gamma_2}(D_0-D(t))-\alpha S_0\exp\left( -\frac{\beta}{\gamma_2}(D(t)-D_0) \right)
\]
which combined with the last equation of \eqref{sird_adj_a} shows that $D(t)$ satisfies the differential equation
\begin{equation}\label{e:D}
\frac{dD}{dt}(t)=\frac{\gamma_2}{\alpha}\tilde{I}_0-(\gamma_1+\gamma_2)(D(t)-D_0) + \gamma_2 S_0\left[ 1-\exp\left( -\frac{\beta}{\gamma_2}(D(t)-D_0) \right)\right].
\end{equation}

Before we move forward, we will treat a little bit a general problem. Assume we take a differential equation of the form
 \[
  \frac{dX}{dt}=f(X) \text{ with } X(0)=X_0\ge 0
 \]
where $f:\R\to\R$ is a Lipschitz function with $f(X_0)>0$.  The solution $X_t$ starts positive, and the derivative is positive, thus the solution is non-decreasing for a while. Moreover, the solution is defined for all $t\ge0$ from general results for ordinary differential equations.  By continuity of the solution, we have that the solution stays in the region $f> 0$ and thus it is increasing for all times it stays inside the region $f>0$.   It can not hit in finite time a point where $f(X(t_c))=0$ since then, reverting the equation (looking at $X(t_c-s)$) and combining this with the uniqueness of the solution, we must have that $X_t=X_{t_c}$ which is a contradiction.  Thus, the solution is increasing and we can integrate the equation as follows:
 \[
  \phi(X(t))=t \text{ where }\phi(x)=\int_{X_0}^x\frac{1}{f(s)}ds.
 \]
Notice here that the function $\phi$ is well defined on the interval of $f>0$ which contains $X_0$. Therefore we have $\phi(X(t))=t$ for all $t\ge0$ and $X(t)$ is increasing from $0$ to infinity.

Now assume that we have two differential equations
\[
\frac{dX(t)}{dt}=f(X(t)) \text{ and }\frac{dY(t)}{dt}=g(Y(t)) \text{ with } X_0=Y_0, f(X_0)>0,g(X_0)>0.
\]
At this stage, the point is that if $X(t_i)=Y(t_i)$ for some sequence of points $0=t_0<t_1<t_2<t_3<t_4$, then we obtain that
\[
 \phi(X(t_i))-\psi(X(t_i))=0 \text{ where }\phi(x)=\int_{X_0}^x\frac{1}{f(s)}ds, \psi(x)=\int_{X_0}^x\frac{1}{g(s)}ds.
\]
In particular, this implies that the function $\phi(x)-\psi(x)$ has at least five zeros.  Since the function $\phi(x)-\psi(x)$ is $C^1$, this implies that the derivative has at least four zeros, in other words this means that $\frac{1}{f(x)}-\frac{1}{g(x)}$ has at least four zeros. Finally, this means that $f(x)=g(x)$ has at least four solutions.

Returning to our problem we take now some parameters $a,b,c,d,\tilde{a},\tilde{b},\tilde{c},\tilde{d}$ and consider
\[
 f(x)=a-bx+c(1-e^{-dx})\text{ while }g(x)=\tilde{a}-\tilde{b}x+\tilde{c}(1-e^{-\tilde{d}x}).
\]
In the case $f(x)-g(x)=0$ has at least four solutions, we actually also get that $f'(x)-g'(x)=0$ has at least three solutions, which then upon taking another derivative gives that $f''(x)-g''(x)=0$ has at least 2 solutions.  Now this means that
\[
 cd^2 e^{-dx}= \tilde{c}\tilde{d}^2e^{-\tilde{d}x}
\]
has at least two different solutions.  The point is that if the above is satisfied for two different values of $x$, say $x_1$ and $x_2$, then
\[
 \frac{cd^2}{\tilde{c}\tilde{d}^2}=e^{(d-\tilde{d})x_1}=e^{(d-\tilde{d})x_2}
\]
which then leads to the conclusion that we must have $d=\tilde{d}$ and $c=\tilde{c}$.   Going now back the ladder, using the fact that $f'(x)=g'(x)$ for three distinct values of $x$ we have
\[
 -b+cde^{-dx}=-\tilde{b}+cde^{-dx}
\]
and thus $b=\tilde{b}$.  Finally, having $f(x)=g(x)$ for five different values of $x$ means that we also get that $a=\tilde{a}$, thus all the parameters must be equal.

Taking this back to our equation \eqref{e:D} and taking $X(t)=D(t)-D_0$, knowing the values $X(0)$, $X(1)$, $X(2)$, $X(3)$, $X(4)$, then we can uniquely determine the values of
\[
\begin{cases}
a=\frac{\gamma_2}{\alpha}\tilde{I}_0 \\
b=\gamma_1+\gamma_2\\
c=\gamma_2 S_0\\
d=\frac{\beta}{\gamma_2}.
\end{cases}
\]
Knowing these values is not enough to determine all the values of $\gamma_1,\gamma_2,\beta,\alpha$ because $S_0$ we know that
\[
 S_0=1-\frac{1}{\alpha}(\tilde{I}_0+\tilde{R}_0)-D_0
\]
which shows that we can solve now
\[
\begin{cases}
\alpha=\frac{c\tilde{I}_0+a(\tilde{I}_0+\tilde{R}_0)}{a(1-D_0)}\\
\beta=\frac{d(c\tilde{I}_0+a(\tilde{I}_0+\tilde{R}_0))}{(1-D_0)\tilde{I}_0}\\
\gamma_1=\frac{(b(1-D_0)-c)\tilde{I}_0-a(\tilde{I}_0+\tilde{R}_0)}{(1-D_0)\tilde{I}_0}\\
 \gamma_2=\frac{c\tilde{I}_0+a(\tilde{I}_0+\tilde{R}_0)}{(1-D_0)\tilde{I}_0}.
\end{cases}
\]
Consequently, knowing $D_0,D_1,D_2,D_3,D_4$ and $\tilde{I}_0,\tilde{R}_0$ we can determine the parameters $\beta,\gamma_1,\gamma_2, \alpha$

\end{proof}

\section{Appendix 2}\label{a:2}

We apply the same methodology to 3 other countries and analyze the results. We use the COVID19 data of Hungary, Czech Republic and Poland. We have to mention that we don’t have information about the procedure of reporting the number of cases for these countries, which can cause a bias in the results.\\ 
By replicating the technique we are confident that the predictive model that we presented in this paper can be also applied to other countries in order to identify the parameters of the model and to accurately assess the transmission dynamic of the pandemic. The results for the 3 countries are detailed below.\\ 

\textbf{1. Hungary}\\

First of all we clean the data and we obtain:
\begin{figure}[H]
  \centering
 \includegraphics[width=\linewidth]{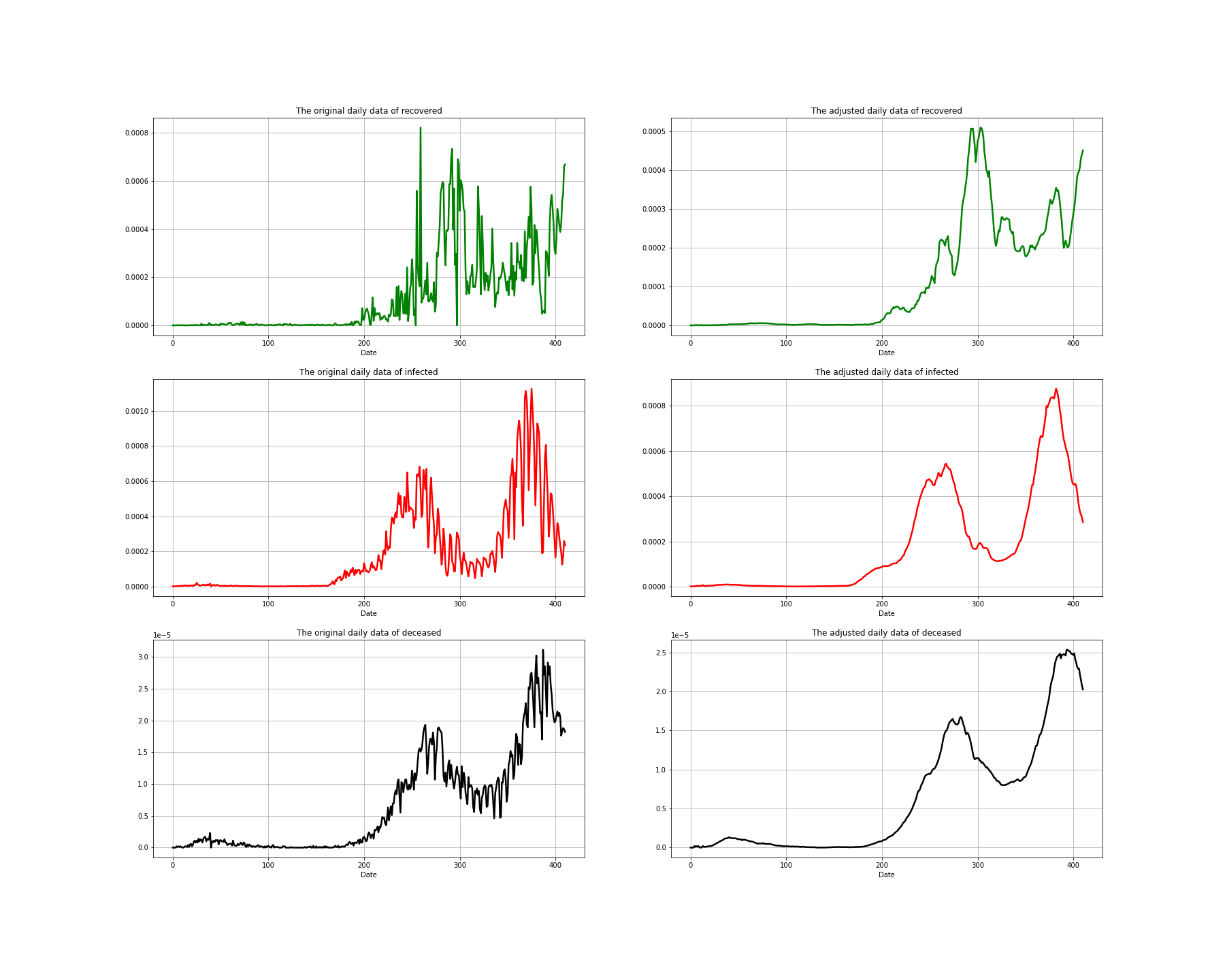}
   \caption{  Data adjustments according to the methodology presented into the article.  The rows describe the data (from top to bottom) of recovered, infected and dead.  The left column represents the raw data and the right column represents the adjusted data as we described above.  Notice the scale and the spike in the first picture which is adjusted as we pointed out.  The data was scaled by $10,000,000$, exactly as in the case of Romania}
\end{figure}

In the next Figure we present the results of the estimated parameters $\beta, \gamma_1,\gamma_2,\alpha$, for Hungary:

\begin{figure}[H]
  \centering
  \begin{subfigure}[b]{0.48\linewidth}
   \includegraphics[width=\linewidth]{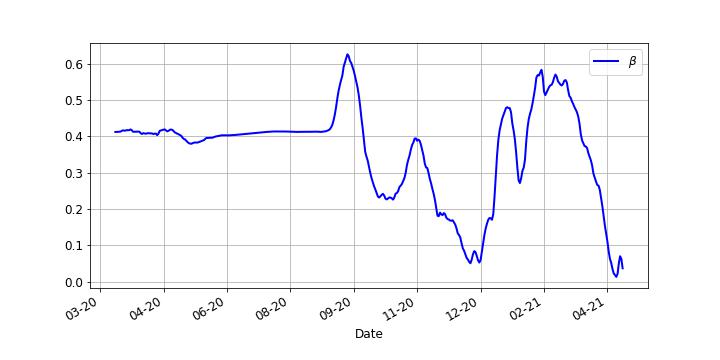}
  \end{subfigure}
  \begin{subfigure}[b]{0.48\linewidth}
   \includegraphics[width=\linewidth]{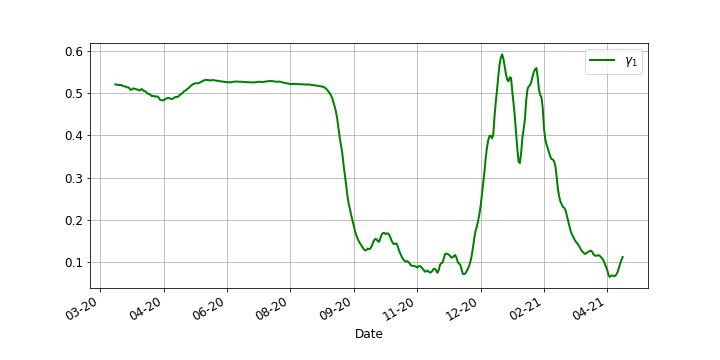}
  \end{subfigure}
  \begin{subfigure}[b]{0.48\linewidth}
   \includegraphics[width=\linewidth]{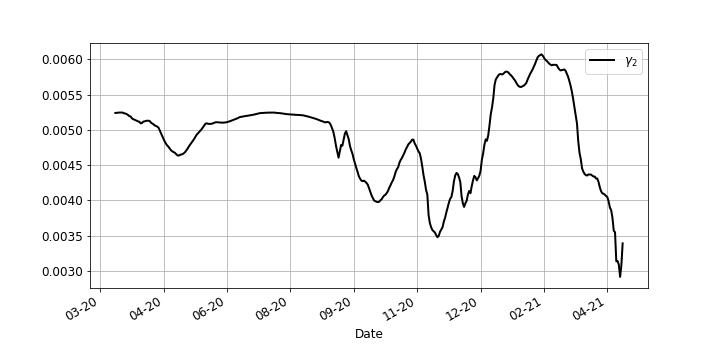}
  \end{subfigure}
  \begin{subfigure}[b]{0.48\linewidth}
   \includegraphics[width=\linewidth]{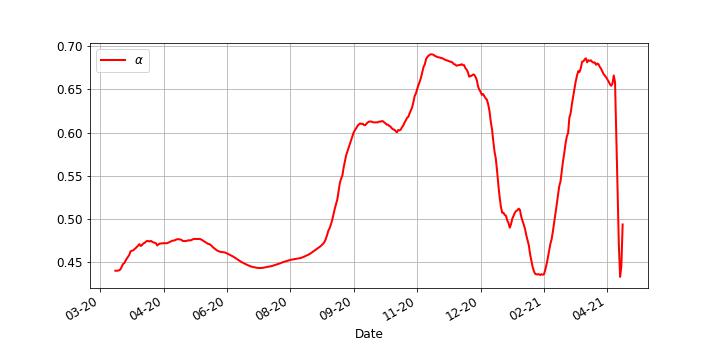}
  \end{subfigure}
  \caption{Parameters of the spreading of Covid19 in Hungary estimated using the same methodology as we used in the case of Romania.}
\end{figure}

Regarding the predictions:\\

1. Prediction of deaths, for 10 days:

\begin{figure}[H]
  \centering
 \includegraphics[width=\linewidth]{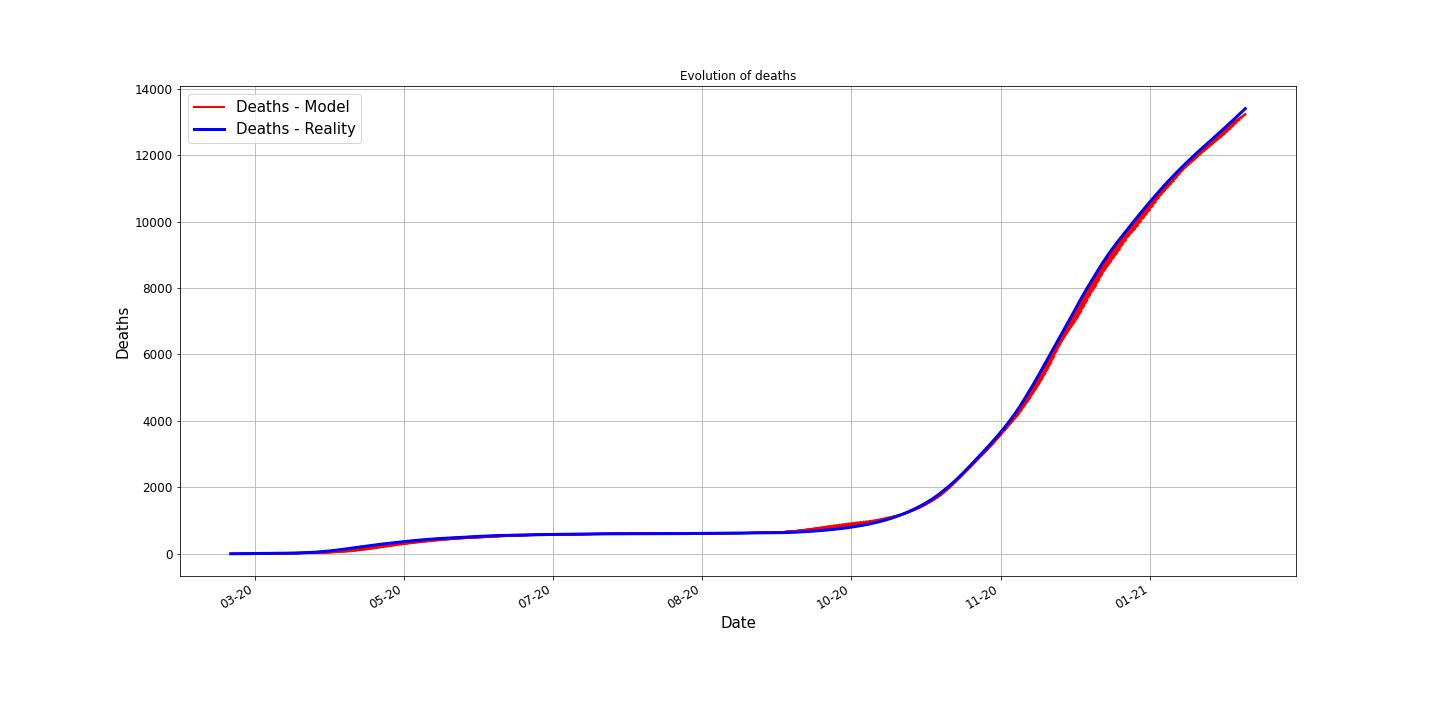}
   \caption{In blue is the real (reported) number of deaths.  For each day $k$, we plotted, in red, the prediction of the deaths, starting with day $k$, for 10 days.}
\end{figure}

2. Prediction of deaths, for 30 days:

\begin{figure}[H]
  \centering
 \includegraphics[width=\linewidth]{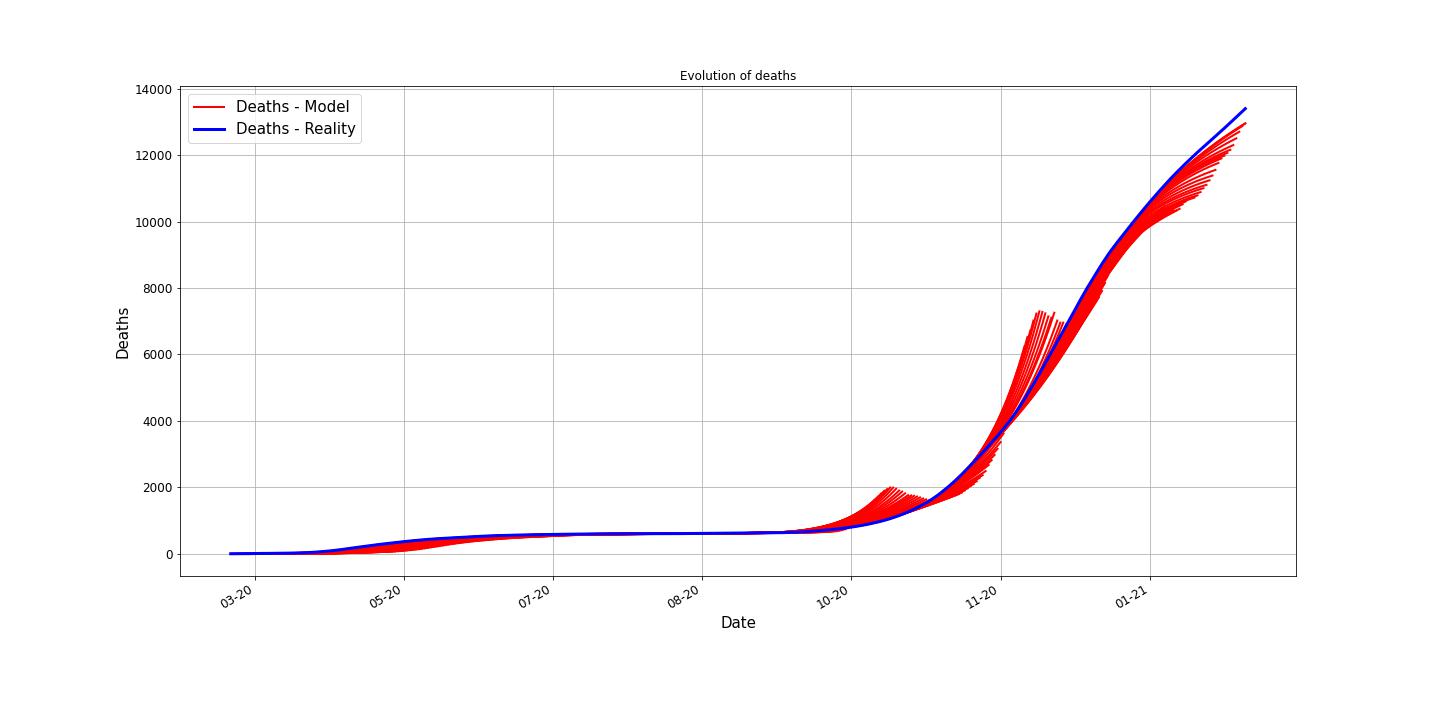}
   \caption{In blue is the real (reported) number of deaths.  For each day $k$, we plotted, in red, the prediction of the deaths, starting with day $k$, for 30 days.}
\end{figure}

3. Prediction of deaths, for 45 days:

\begin{figure}[H]
  \centering
 \includegraphics[width=\linewidth]{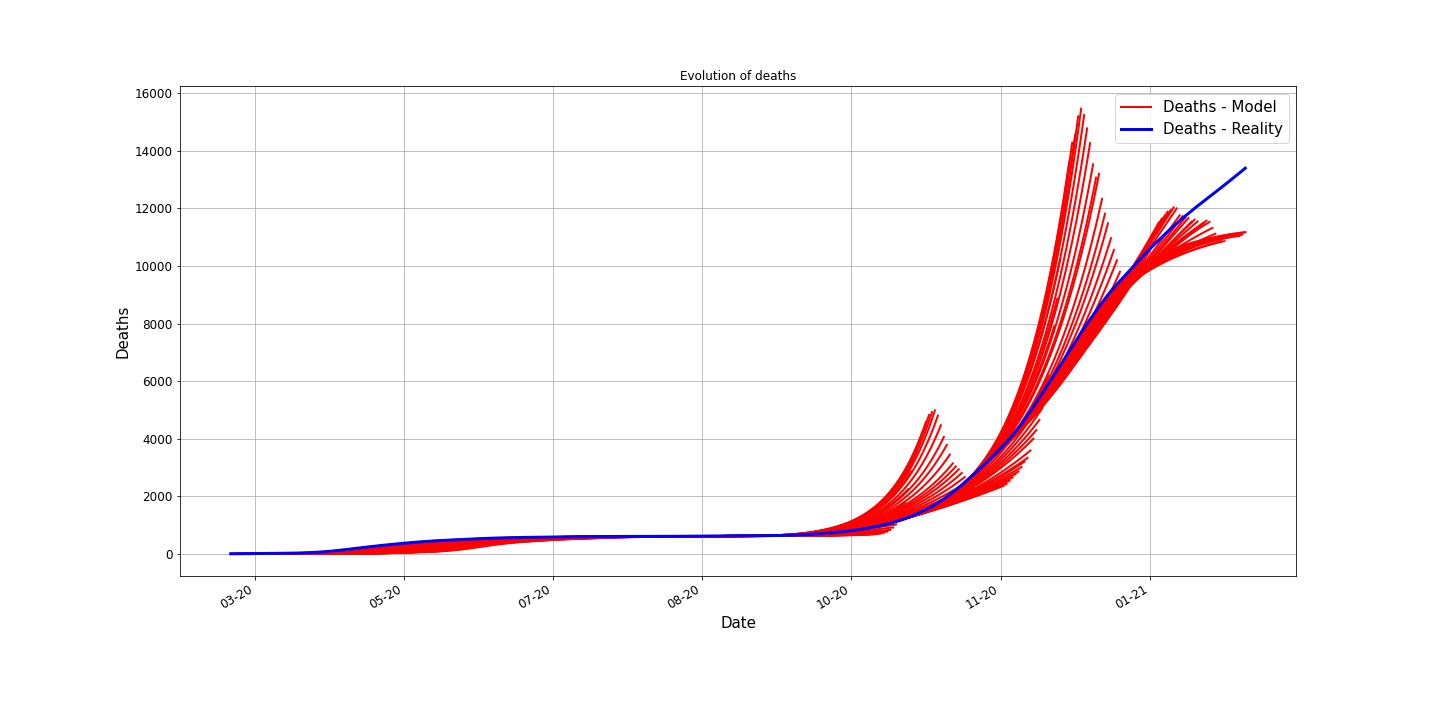}
   \caption{In blue is the real (reported) number of deaths.  For each day $k$, we plotted, in red, the prediction of the deaths, starting with day $k$, for 45 days.}
\end{figure}

Regarding the MAE values, for Hungary:
\begin{center}
\begin{tabular}{ |c|c|c|c| } 
\hline
Prediction & Case & MAE \\
\hline
\multirow{3}{*}{10 days prediction } & Deaths & 40.36850 \\ 
& Infected  & 699.77421 \\ 
& Recovered & 723.19638 \\ 
\hline
\multirow{3}{*}{30 days prediction } & Deaths & 153.05402 \\ 
& Infected  & 6600.95112 \\ 
& Recovered & 3645.40171 \\ 
\hline
\multirow{3}{*}{45 days prediction } & Deaths & 309.09216 \\ 
& Infected  & 15741.09146 \\ 
& Recovered & 6408.25157 \\ 
\hline
\end{tabular}
\end{center}

\textbf{2. Czech Republic}\\

First of all we clean the data and we obtain:
\begin{figure}[H]
  \centering
 \includegraphics[width=\linewidth]{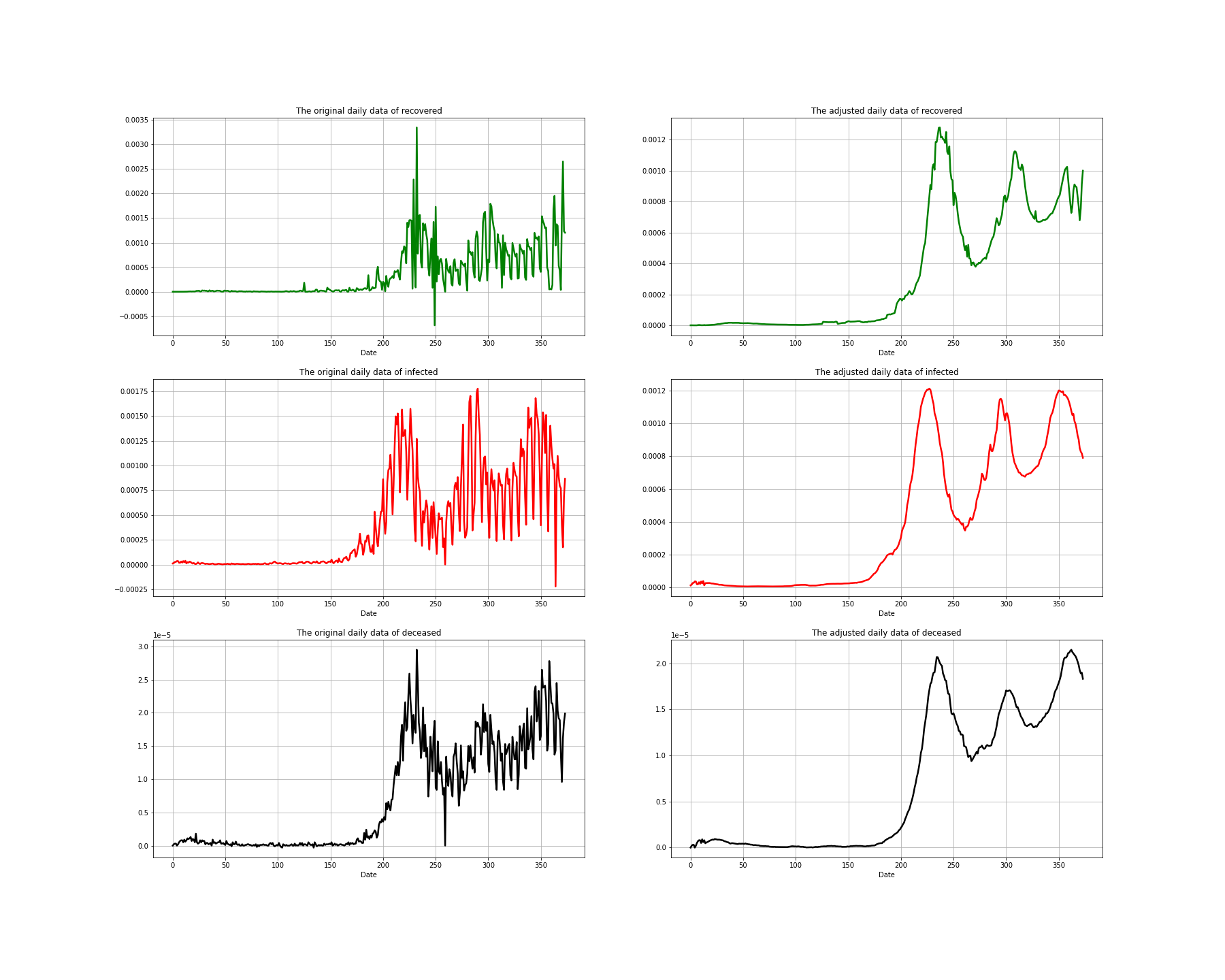}
   \caption{  Data adjustments according to the methodology presented into the article.  The rows describe the data (from top to bottom) of recovered, infected and dead.  The left column represents the raw data and the right column represents the adjusted data as we described above.  Notice the scale and the spike in the first picture which is adjusted as we pointed out.  The data was scaled by $10,000,000$, exactly as in the case of Romania}
\end{figure}

In the next Figure we present the results of the estimated parameters $\beta, \gamma_1,\gamma_2,\alpha$, for Czech Republic:

\begin{figure}[H]
  \centering
  \begin{subfigure}[b]{0.48\linewidth}
   \includegraphics[width=\linewidth]{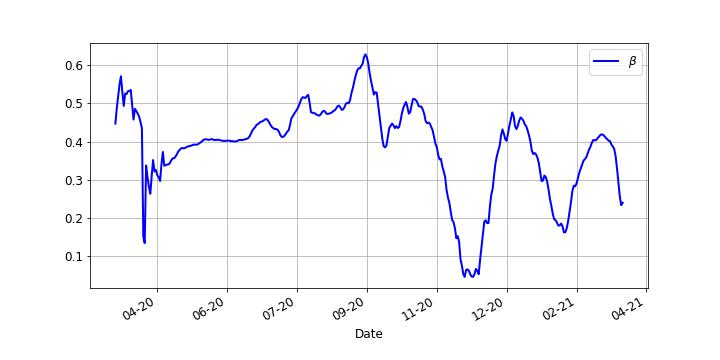}
  \end{subfigure}
  \begin{subfigure}[b]{0.48\linewidth}
   \includegraphics[width=\linewidth]{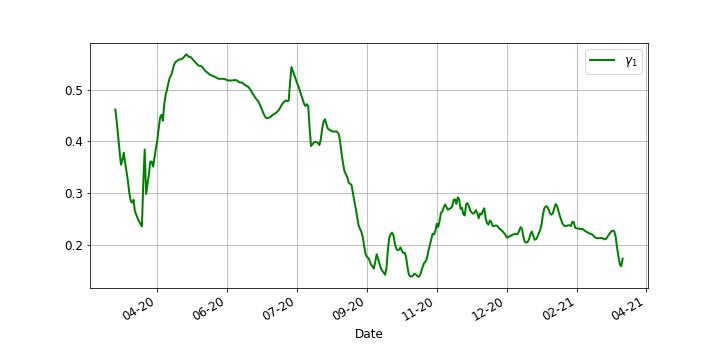}
  \end{subfigure}
  \begin{subfigure}[b]{0.48\linewidth}
   \includegraphics[width=\linewidth]{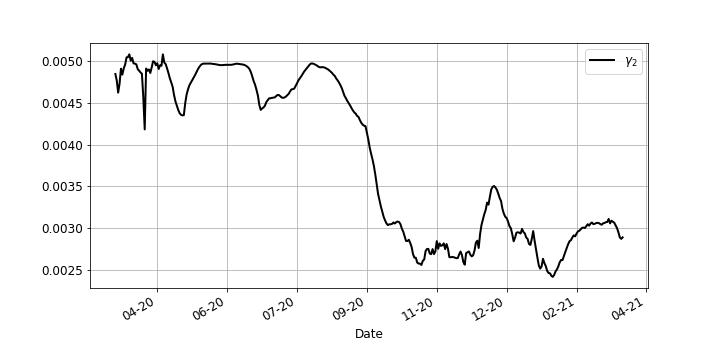}
  \end{subfigure}
  \begin{subfigure}[b]{0.48\linewidth}
   \includegraphics[width=\linewidth]{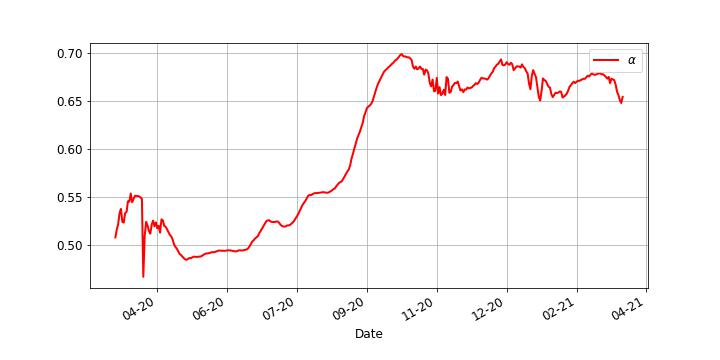}
  \end{subfigure}
  \caption{Parameters of the spreading of Covid19 in Czech Republic estimated using the same methodology as we used in the case of Romania.}
\end{figure}

Regarding the predictions:\\

1. Prediction of deaths, for 10 days:

\begin{figure}[H]
  \centering
 \includegraphics[width=\linewidth]{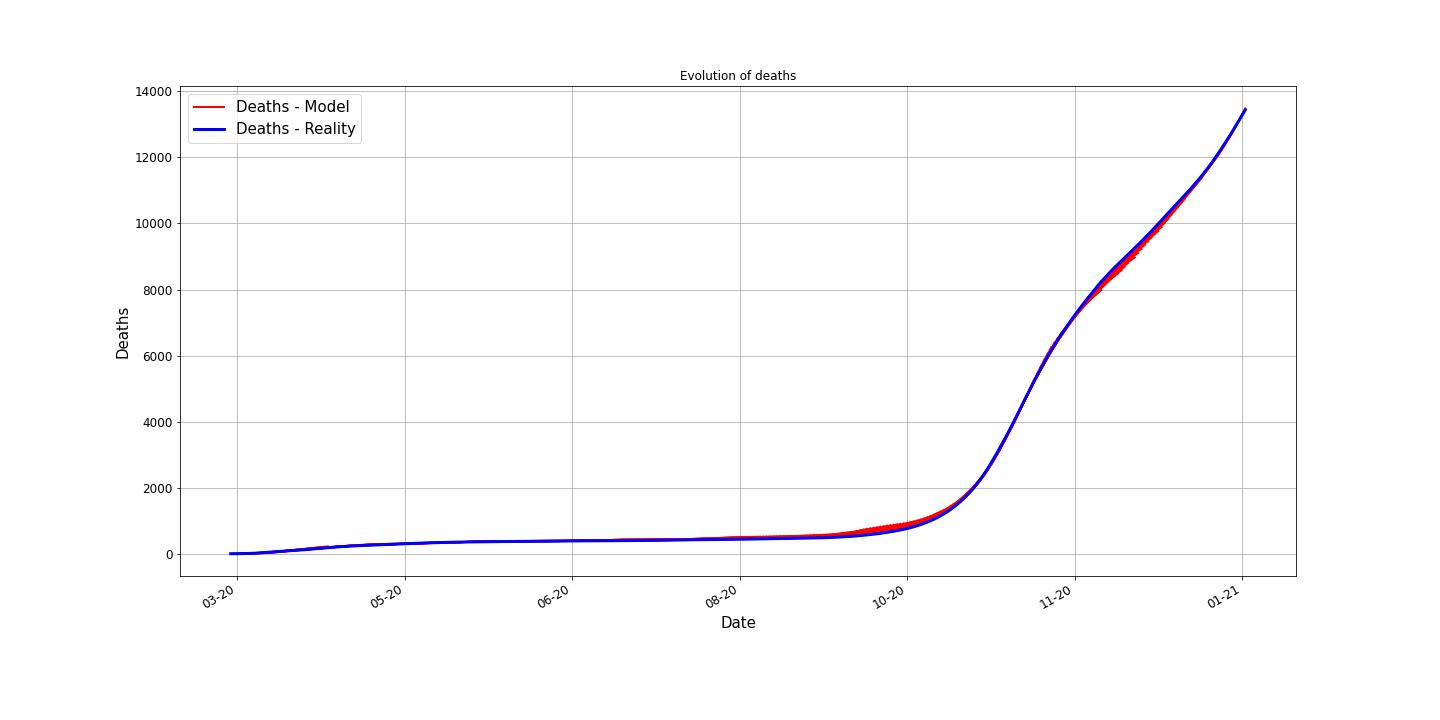}
   \caption{In blue is the real (reported) number of deaths.  For each day $k$, we plotted, in red, the prediction of the deaths, starting with day $k$, for 10 days.}
\end{figure}

2. Prediction of deaths, for 30 days:

\begin{figure}[H]
  \centering
 \includegraphics[width=\linewidth]{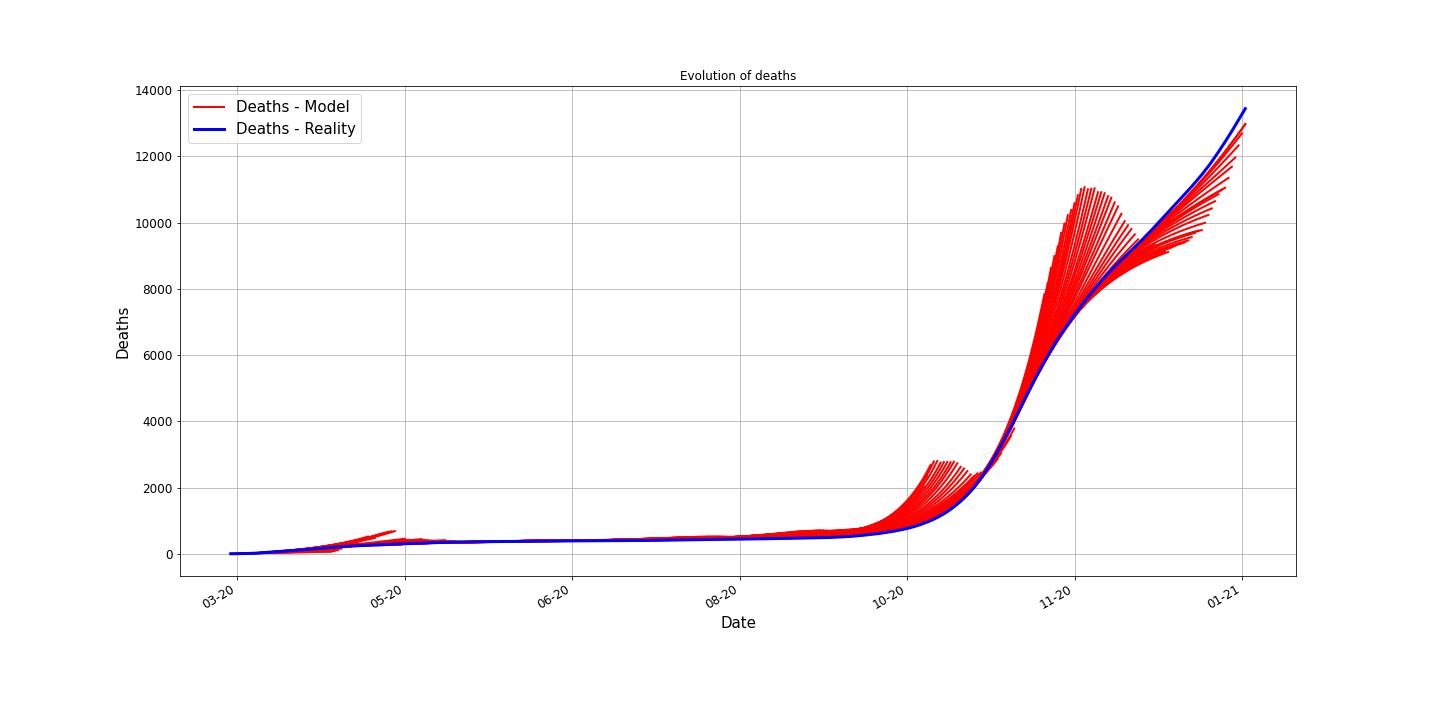}
   \caption{In blue is the real (reported) number of deaths.  For each day $k$, we plotted, in red, the prediction of the deaths, starting with day $k$, for 30 days.}
\end{figure}

3. Prediction of deaths, for 45 days:

\begin{figure}[H]
  \centering
 \includegraphics[width=\linewidth]{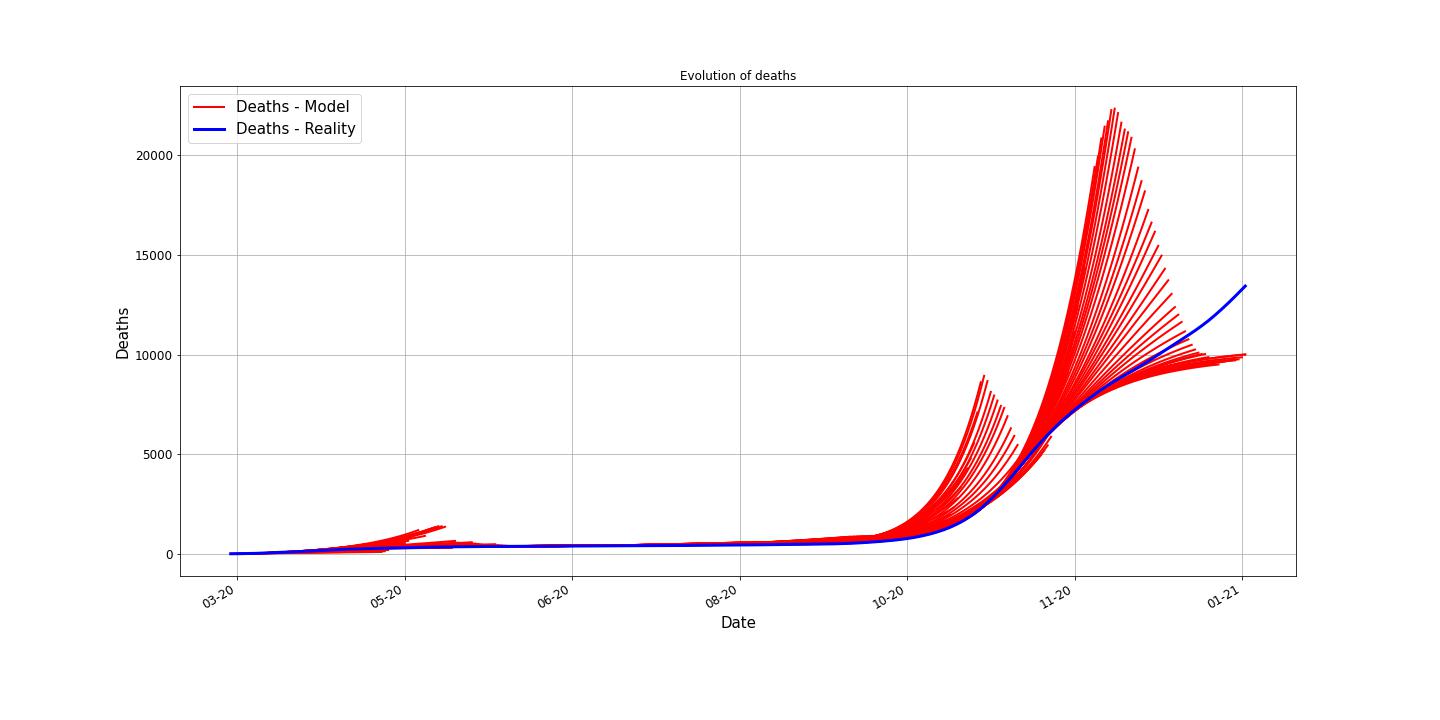}
   \caption{In blue is the real (reported) number of deaths.  For each day $k$, we plotted, in red, the prediction of the deaths, starting with day $k$, for 45 days.}
\end{figure}

Regarding the MAE values, for Czech Republic:
\begin{center}
\begin{tabular}{ |c|c|c|c| } 
\hline
Prediction & Case & MAE \\
\hline
\multirow{3}{*}{10 days prediction } & Deaths & 30.50940 \\ 
& Infected  & 2172.68655 \\ 
& Recovered & 921.83768 \\ 
\hline
\multirow{3}{*}{30 days prediction } & Deaths & 189.93063 \\ 
& Infected  & 5308.06519 \\ 
& Recovered & 3645.40171 \\ 
\hline
\multirow{3}{*}{45 days prediction } & Deaths & 543.08844 \\ 
& Infected  & 49255.44366 \\ 
& Recovered & 17283.88766 \\ 
\hline
\end{tabular}
\end{center}

\textbf{3. Poland} 

First of all we clean the data and we obtain:
\begin{figure}[H]
  \centering
 \includegraphics[width=\linewidth]{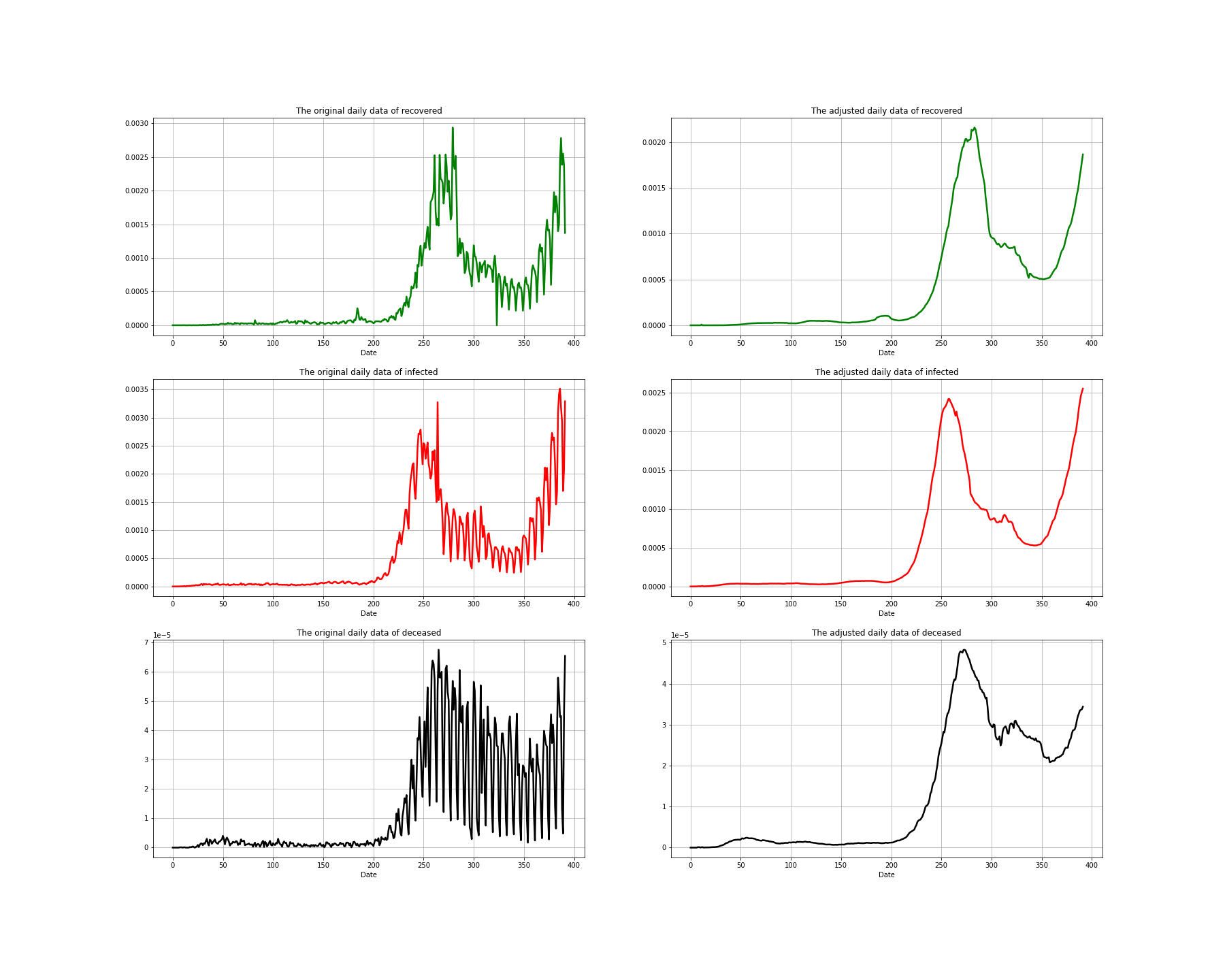}
   \caption{  Data adjustments according to the methodology presented into the article.  The rows describe the data (from top to bottom) of recovered, infected and dead.  The left column represents the raw data and the right column represents the adjusted data as we described above.  Notice the scale and the spike in the first picture which is adjusted as we pointed out.  The data was scaled by $10,000,000$, exactly as in the case of Romania}
\end{figure}

In the next Figure we present the results of the estimated parameters $\beta, \gamma_1,\gamma_2,\alpha$, for Poland:

\begin{figure}[H]
  \centering
  \begin{subfigure}[b]{0.48\linewidth}
   \includegraphics[width=\linewidth]{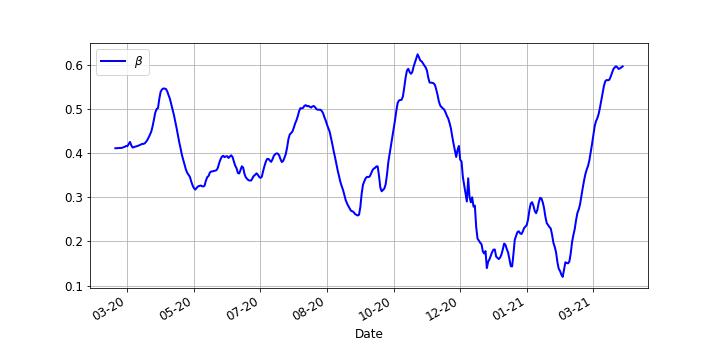}
  \end{subfigure}
  \begin{subfigure}[b]{0.48\linewidth}
   \includegraphics[width=\linewidth]{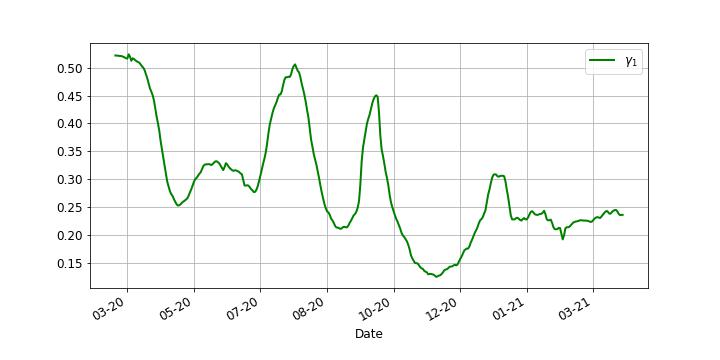}
  \end{subfigure}
  \begin{subfigure}[b]{0.48\linewidth}
   \includegraphics[width=\linewidth]{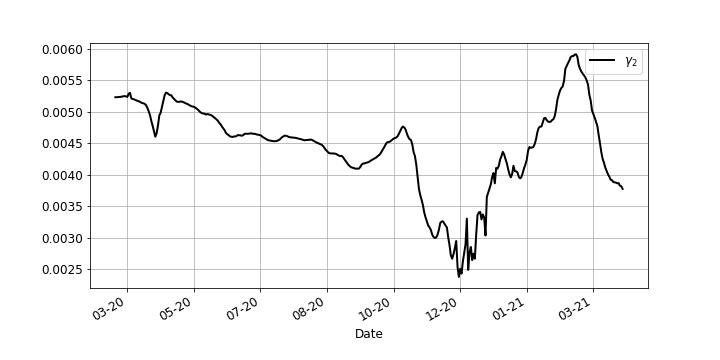}
  \end{subfigure}
  \begin{subfigure}[b]{0.48\linewidth}
   \includegraphics[width=\linewidth]{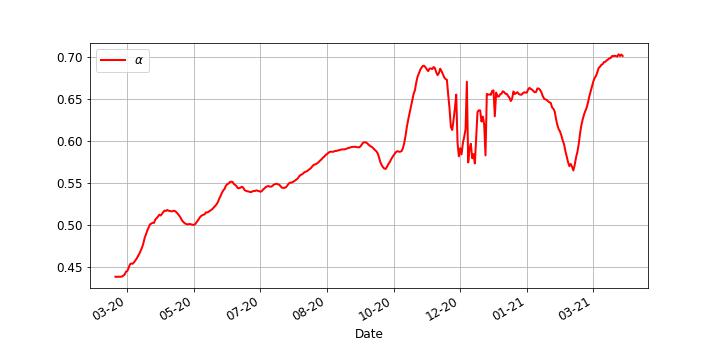}
  \end{subfigure}
  \caption{Parameters of the spreading of Covid19 in Poland estimated using the same methodology as we used in the case of Romania.}

\end{figure}

Regarding the predictions:\\

1. Prediction of deaths, for 10 days:

\begin{figure}[H]
  \centering
 \includegraphics[width=\linewidth]{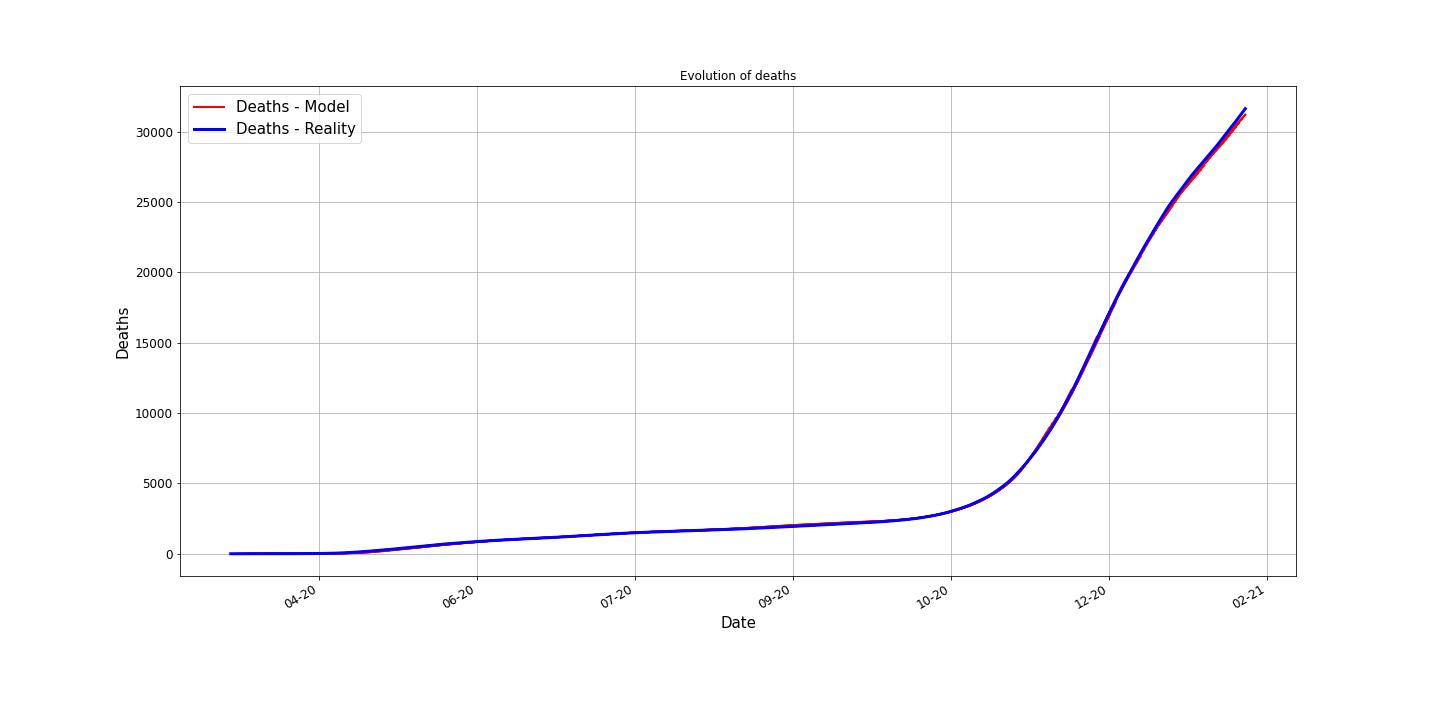}
   \caption{In blue is the real (reported) number of deaths.  For each day $k$, we plotted, in red, the prediction of the deaths, starting with day $k$, for 10 days.}
\end{figure}

2. Prediction of deaths, for 30 days:

\begin{figure}[H]
  \centering
 \includegraphics[width=\linewidth]{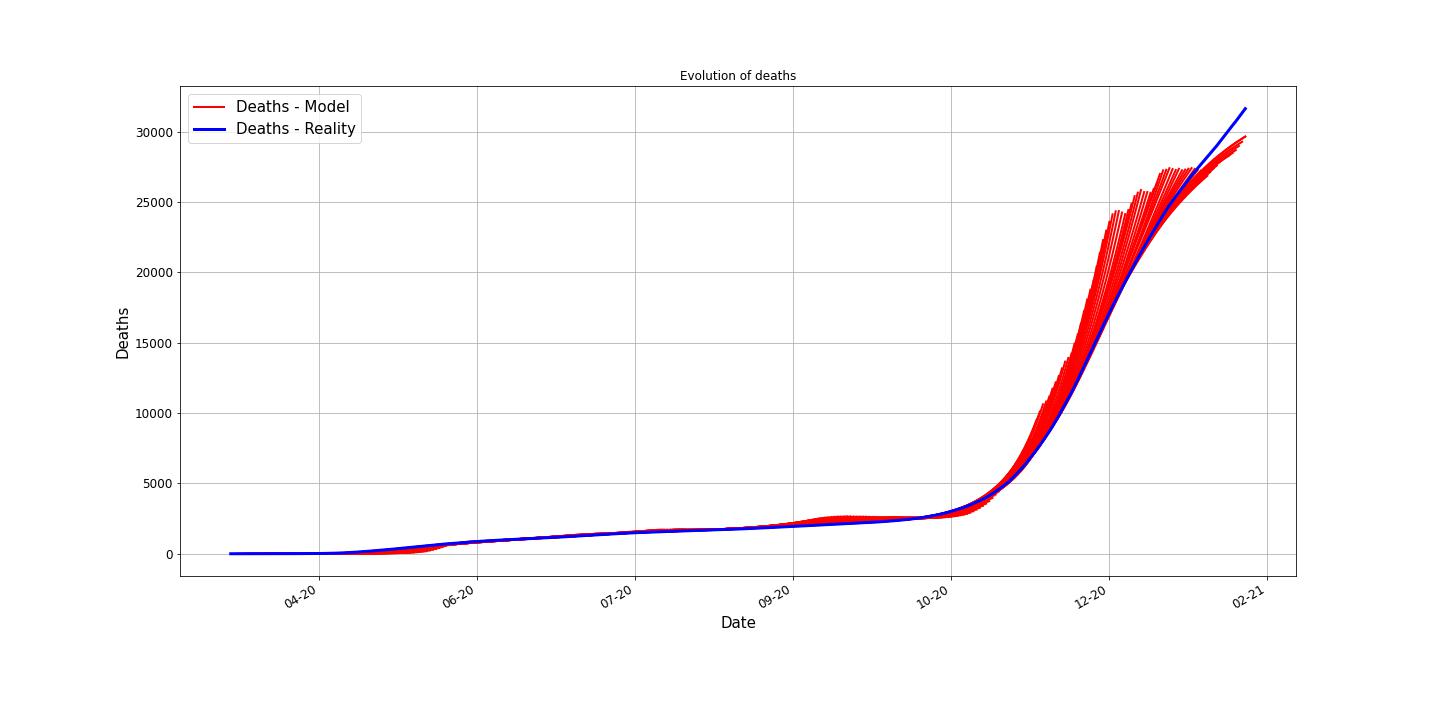}
   \caption{In blue is the real (reported) number of deaths.  For each day $k$, we plotted, in red, the prediction of the deaths, starting with day $k$, for 30 days.}
\end{figure}

3. Prediction of deaths, for 45 days:

\begin{figure}[H]
  \centering
 \includegraphics[width=\linewidth]{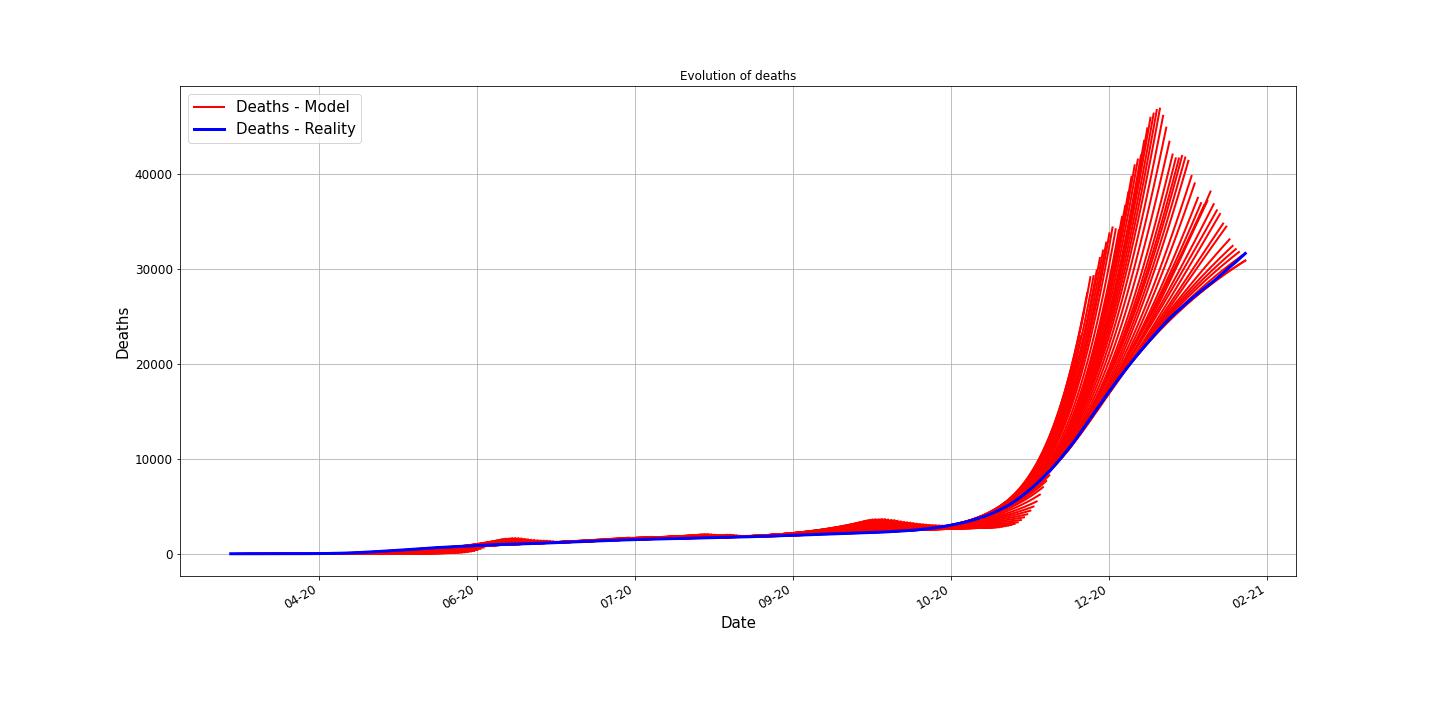}
   \caption{In blue is the real (reported) number of deaths.  For each day $k$, we plotted, in red, the prediction of the deaths, starting with day $k$, for 45 days.}
\end{figure}

Regarding the MAE values, for Poland:
\begin{center}
\begin{tabular}{ |c|c|c|c| } 
\hline
Prediction & Case & MAE \\
\hline
\multirow{3}{*}{10 days prediction } & Deaths & 42.08050 \\ 
& Infected  & 3397.06115 \\ 
& Recovered & 1122.41742 \\ 
\hline
\multirow{3}{*}{30 days prediction } & Deaths & 270.48052 \\ 
& Infected  & 30818.55274 \\ 
& Recovered & 5515.75961 \\ 
\hline
\multirow{3}{*}{45 days prediction } & Deaths & 842.45188 \\ 
& Infected  & 72149.78502 \\ 
& Recovered & 16057.94535 \\ 
\hline
\end{tabular}
\end{center}

\end{document}